\documentclass{article}

\usepackage[preprint,nonatbib]{neurips_2022}

\usepackage[utf8]{inputenc} %
\usepackage[T1]{fontenc}    %
\usepackage{hyperref}       %
\usepackage{url}            %
\usepackage{booktabs}       %
\usepackage{amsfonts}       %
\usepackage{nicefrac}       %
\usepackage{microtype}      %
\usepackage{xcolor}         %
\usepackage[numbers]{natbib}

\usepackage{xparse}
\usepackage{macros}%

\title{A Unifying Framework for Online Optimization with Long-Term Constraints}

\author{%
	Matteo Castiglioni\thanks{DEIB, Politecnico di Milano, \{matteo.castiglioni, alberto.marchesi, giulia.romano, nicola.gatti\}@polimi.it. $^\dagger$ Computing Sciences Department, Bocconi University, andrea.celli2@unibocconi.it.}\\
	Politecnico di Milano\\
	\And 
	Andrea Celli$^\dagger$\\
	Bocconi University\\
	\And 
	Alberto Marchesi$^\ast$\\
	Politecnico di Milano\\
	\And
	Giulia Romano$^\ast$\\
	Politecnico di Milano\\
	\And
	Nicola Gatti$^\ast$\\
	Politecnico di Milano
}

\begin{document}
	
\maketitle

\begin{abstract}
We study online learning problems in which a decision maker has to take a sequence of decisions subject to $m$ \emph{long-term constraints}. The goal of the decision maker is to maximize their total reward, while at the same time achieving small cumulative constraints violation across the $T$ rounds. 
We present the first \emph{best-of-both-world} type algorithm for this general class of problems, with no-regret guarantees both in the case in which rewards and constraints are selected according to an unknown stochastic model, and in the case in which they are selected at each round by an adversary.
Our algorithm is the first to provide guarantees in the adversarial setting with respect to the optimal fixed strategy that satisfies the long-term constraints. In particular, it guarantees a $\rho/(1+\rho)$ fraction of the optimal reward and sublinear regret, where $\rho$ is a feasibility parameter related to the existence of strictly feasible solutions.
Our framework employs traditional regret minimizers as black-box components. Therefore, by instantiating it with an appropriate choice of regret minimizers it can handle the \emph{full-feedback} as well as the \emph{bandit-feedback} setting. 
Moreover, it allows the decision maker to seamlessly handle scenarios with non-convex rewards and constraints. 
We show how our framework can be applied in the context of budget-management mechanisms for repeated auctions in order to guarantee long-term constraints that are not \emph{packing} (\emph{e.g.}, ROI constraints). 
\end{abstract}

\section{Introduction}

We study online learning problems where a decision maker takes decisions over $T$ rounds.
At each round $t$, the decision $\vx_t\in\cX$ is chosen before observing a reward function $f_t$ together with a set of $m$ \emph{time-varying} constraint functions $g_t$.
The decision maker is allowed to make decisions that are \emph{not} feasible, provided that the overall sequence of decisions obeys the \emph{long-term constraints} $\sum_{t=1}^T g_t(\vx_t)\le\mathbf{0}$, up to a small cumulative violation across the $T$ rounds.
The problem becomes that of finding a sequence of decisions $\vx_t$ which guarantees a reward close to that of the best fixed decision in hindsight while satisfying long-term constraints.
This type of framework was first proposed by~\citet{mannor2009online},
and it has numerous applications ranging from wireless communication~\citep{mannor2009online} and multi-objective online classification~\citep{bernstein2010online}, to \emph{safe} online learning~\citep{amodei2016concrete}.

\citet{mannor2009online}~show that guaranteeing sublinear regret
and sublinear cumulative constraints violation is impossible even when $f_t$ and $g_t$ are simple linear functions.
Therefore, previous works either focus on the case in which constraints are generated i.i.d.~according to some unknown stochastic model, without providing any guarantees for the adversarial case, or provide results for adversarially-generated constraints under some strong assumptions on the structure of the problem or using a weaker baseline (a detailed discussion of related works can be found in \Cref{sec:related}).
A few examples in the latter case are \cite{sun2017safety,yi2020distributed,chen2017online,cao2018online}.
In the former setting (\emph{i.e.}, stochastic constraints), \citet{wei2020online} consider a weaker baseline that is feasible for each constraint $g_t$, going against the basic idea of long-term constraints.
A notable exception is the work by~\citet{yu2017online}, who employ the same baseline as ours, and provide an upper bound of $\tilde O(T^{1/2})$ for both regret and constraints violation (see Table~\ref{tab:results}).
We also mention that there are some works studying the problem in which constraints are \emph{static} (see, \emph{e.g.}, \cite{jenatton2016adaptive,mahdavi2012trading,yu2020low,yuan2018online}), or focus on specific types of constraints, such as \emph{knapsack constraints} \cite{badanidiyuru2018bandits,immorlica2019adversarial}. Our framework differs from those works as we deal with \emph{arbitrary} and \emph{time-varying} constraints.
Moreover, it also extends the \emph{online convex optimization} framework introduced by~\citet{zinkevich2003online} by allowing for general non-convex loss functions $f_t$, arbitrary feasibility sets $\cX$, and arbitrary time-varying long-term constraints.

\begin{centering}
	\begin{table*}
		\sisetup{detect-all=true,scientific-notation=fixed,fixed-exponent=0,round-mode=places,
			round-precision=4}
		\setlength{\tabcolsep}{3pt}\small\centering
		\begin{tikzpicture}
		\node[anchor=south west] at (0, 0) {
			\begin{tabular}{ccccccc}
			\toprule
			\multirow{2}{*}{\bf Algorithm} & \multirow{2}{*}{\bf Constr.} & {\bf Non-convex} & \multicolumn{2}{c}{\bf Bound --- constant $\rho$} & \multicolumn{2}{c}{\bf Bound --- arbitrary $\rho$}\\
			\multicolumn{1}{c}{}& \multicolumn{1}{c}{} & $f_t$ {\bf and} $g_t$ &  Reward & Violation & Reward & Violation \\
			\midrule 
			\citet{yu2017online} & \textsc{Stoc} & \xmark &  $\OPT-\tilde O(T^{1/2})$ & $\tilde O(T^{1/2})$ & \unk & \unk \\
			\midrule
			\multirow{2}{*}{\begin{minipage}{.6cm}\centering \bf Ours\end{minipage}}
			& \textsc{Stoc}   &    \cmark  &  $\OPT-\tilde O(T^{1/2})$ &  $\tilde O(T^{1/2})$ &  \cellcolor{gray!20} $\OPT-\tilde O(T^{3/4})$&  \cellcolor{gray!20} $\tilde O(T^{3/4})$\\
			& \textsc{Adv}    &   \cmark  & \cellcolor{gray!20} $\frac{\rho}{1+\rho} \OPT - \tilde{O}\mleft(T^{1/2}\mright)$ &  \cellcolor{gray!20} $\tilde O(T^{1/2})$ & \unk & \unk \\
			\bottomrule
			\end{tabular}};
		\end{tikzpicture}\vspace{-3mm}
		\caption{ Comparison between the performance of our algorithm and previous work using the same baseline as ours. Bounds for settings that were \emph{not} previously tractable are highlighted in gray. $\OPT$ is the reward of the baseline.
		}
		\label{tab:results}
	\end{table*}
\end{centering}

\subsection{Original contributions}

Given the negative result by \citet{mannor2009online}, a natural question is what kind of guarantees we can reach in the adversarial setting, when adopting the standard baseline of the best fixed decision in hindsight satisfying (in expectation) the long-term constraints.
We provide the first positive result going in this direction, by designing a no-$\alpha$-regret algorithm that guarantees a sublinear cumulative constraints violation.
Moreover, we make a step forward in the line of work initiated by \citet{bubeck2012best}, by showing that our algorithm is also the first \emph{best-of-both-worlds} algorithm for problems with arbitrary long-term constraints. This allows our algorithm to guarantee good worst-case performance (adversarial case), while being able to exploit well-behaved problem instances (stochastic case).
The only assumption which we require is the existence of a decision that is strictly feasible with respect to the sequence of constraints. We denote by $\rho$ the ``margin'' by which this decision is strictly feasible (see Section~\ref{sec:prel} for a definition).
At the same time, we show that even without this assumption, we can recover sublinear regret and violation with stochastic constraints.

Previous work usually assumes that $\rho$ is a given \emph{constant}.
In that case, our algorithm matches the guarantees by~\citet{yu2017online} when constraints are generated i.i.d.~according to an unknown distribution, and has no-$\alpha$-regret with $\alpha=\rho/(1+\rho)$ in the adversarial case (see Table~\ref{tab:results}).
Our algorithm only requires a lower bound on the real value of the feasibility parameter $\rho$.
In the stochastic case, the lower bound may even be unknown, and the algorithm can efficiently estimate it from data.  
Moreover, we argue that if $\rho$ is allowed to depend on $T$ and take arbitrarily small values, then there are certain values ($\rho\le T^{-1/4}$), for which any regret bound depending on $1/\rho$ would be useless (\emph{i.e.}, \emph{not} sublinear in $T$, see Section~\ref{sec:alg}).
This setting is usually overlooked by previous work, which assumes $\rho$ to be a given constant. We show that, in the case of an arbitrary feasibility parameter $\rho$, in the stochastic setting our algorithm guarantees an upper bound of $\tilde O(T^{3/4})$ for regret and cumulative constraints violation.  

Our framework employs traditional regret minimizers as black-box components. Therefore, by instantiating it with an appropriate choice of regret minimizers it can handle \emph{full-feedback} as well as \emph{bandit-feedback} settings. In the former case, after playing $\vx_t$, the decision maker gets to observe $f_t$ and $g_t$, while in the latter case only the realized values $f_t(\vx_t)$ and $g_t(\vx_t)$ are observed.
Moreover, this allows the decision maker to seamlessly handle scenarios with non-convex reward and constraints, by employing a suitable regret minimizer for non-convex losses (see, \emph{e.g.}, \cite{suggala2020Online}).
Our algorithm is based on a two-stage approach in which \emph{primal} and \emph{dual} players interact through \emph{Lagrangian games}.
In the first (\emph{play}) phase, the primal player tries to balance out the maximization of their rewards with constraints violation.
In the second (\emph{recovery}) phase, the primal player only makes ``safe decisions'' to avoid violating constraints too much.
It is possible to prove that, in the case of stochastic rewards and constraints, the algorithm never enters phase two.
This property is particularly relevant for budget-pacing mechanisms in repeated auctions, since it is related to how budget is allocated.
Our framework can also be instantiated to perform budget allocation subject to constraints that were previously \emph{not} tractable by traditional mechanisms, such as ROI constraints~\cite{balseiro2019learning,conitzer2021multiplicative}.

\section{Related works}\label{sec:related}

The \emph{online convex optimization} (OCO) framework was first proposed in the machine learning literature by \citet{zinkevich2003online}, and since then it has significantly expanded becoming widely influential in the learning community (see, \emph{e.g.}, \cite{hazan2006efficient,hazan2016introduction,shalev2012online}).
In what follows, we highlight the most relevant works with respect to ours from the literature related to online convex optimization problems with constraints. The analysis and the results are quite different depending on the nature of the constraints, which may be static, \emph{i.e.}, time-invariant, or stochastic/adversarial, \emph{i.e.}, time-variant.

\paragraph{Static constraints.}
\citet{zinkevich2003online} first addressed online convex optimization problems with static constraints by developing a projection-based \emph{online gradient descent} (OGD) algorithm. This method guarantees a regret upper bound of $\mathcal{O}(\sqrt{T})$ for an arbitrary sequence of convex objective functions with bounded subgradients.
\citet{hazan2007logarithmic} showed that this is a tight bound up to constant factors. 
When the set defined by the static constraints is complex, the conventional projection-based online algorithms can be difficult to implement due to the potentially high computational cost of carrying out the projection operation.
To overcome this difficulty, \citet{mahdavi2012trading} propose an efficient algorithm which is an adaptation of OGD achieving a cumulative regret of order $\mathcal{O}(\sqrt{T})$ and a cumulative constraints violation of $\mathcal{O}(T^{{3}/{4}})$.
These bounds are generalized by \citet{jenatton2016adaptive} who propose an algorithm that achieves a cumulative regret of $\mathcal{O}(T ^{\max\{\beta,1-\beta\}})$ and a cumulative violation of $\mathcal{O}(T ^{1-{\beta}/{2}})$, where $\beta \in (0,1)$ is a user-defined parameter.
Other works, such as, \emph{e.g}, \cite{yuan2018online, yu2020low}, propose primal-dual algorithms and achieve better bounds by making further assumptions. 
In particular, \citet{yu2020low} achieve bounds on the cumulative regret of $\mathcal{O}(\sqrt{T})$ and on the cumulative violation of $\mathcal{O}(1)$ by assuming that the Slater's condition holds (\emph{i.e.}, the existence of a strictly feasible solution). 
Then, \citet{yuan2018online} achieve a cumulative regret of $\mathcal{O}(\log{T})$ and a constraint violation of $\mathcal{O}(\sqrt{T})$ under the assumption that the objective functions are strongly convex.
In all the the papers cited above, the regret is computed with respect to the best fixed action in hindsight,that does \emph{not} violate the constraints at each round $t$. This metric is called \emph{static regret}.

\paragraph{Stochastic constraints.}
\citet{yu2017online} consider an online convex optimization framework with stochastic constraints, where the objective functions are chosen by an adversary, and the constraint functions are independent and identically distributed (i.i.d.) over time.
\citet{yu2017online} provide a primal-dual proximal gradient algorithm achieving $\mathcal{O}(\sqrt{T})$ cumulative regret and constraint violation by assuming Slater's condition.
Moreover, \citet{wei2020online} provide bounds of the same order by assuming a less stringent version of the Slater's condition.
As a performance metric, the latter work use \emph{static regret}.

\paragraph{Adversarial constraints.}
Various works in the literature address the online learning setting with adversarial reward and constraint functions. 
This problem was first studied by \citet{mannor2009online} in a two-player game setting. 
The regret is computed with respect to the best strategy from the set of fixed strategies that satisfy the constraints on average. 
\citet{mannor2009online} show that in general it is impossible to compete against the best decision in such a set. 
In particular, they construct a two-player game where there exists a policy for the adversary such that,  among the policies of the player that violate sublinearly the constraints, there is no policy that can achieve the no-regret property in terms of maximizing the player’s reward. 
\citet{sun2017safety} study a similar problem related to contextual bandits and show that also in their setting the decision maker is unable to compete again this baseline by adapting the result from \citet{mannor2009online} to their setting.
To circumvent this issue and provide some guarantees, they rely on a weaker baseline to compute the regret.
In particular, they assume that the decision set is rich enough that, in hindsight, there exist a fixed action that satisfies the constraints at each round: they are using the \emph{static regret} as a performance metric.
Then, by employing static regret as a baseline, \citet{sun2017safety} show that the approaches of \citet{mahdavi2012trading} and \citet{jenatton2016adaptive} can be extended to the online learning framework with adversarial sequential constraints.
Therefore, they provide an algorithm which is a generalization of that from \citet{mahdavi2012trading} achieving sublinear cumulative regret and constraint violations.

\citet{liakopoulos2019cautious} define a new notion of regret, to overcome the impossibility result from \citet{mannor2009online}.  
They introduce a refined regret metric which compares the agent’s incurred losses to those of a \emph{K-benchmark}, which is the best strategy in the hindsight such that, for each time window of length $K$, the long-term constraints over that window are satisfied.
They provide parametric results that are useful to balance the trade-off between regret minimization and long-term residual constraint violation.
Moreover, instead of the Slater's condition they consider a less stringent assumption related to the definition of their regret metric.

A recent line of works such as \cite{chen2017online,chen2018bandit} and \cite{cao2018online} provide some results related to the regret against \emph{dynamic policies}. 
As expected, comparing against a dynamic baseline require very strong assumptions.
\citet{chen2017online} compute a bound on the cumulative dynamic regret which is sublinear in the time horizon $T$ only if the drift of the baseline sequence (\emph{i.e.}, $\sum_{t=1}^T ||\bm{x}^*_{t+1} - \bm{x}^*_t ||$) and that of the constraints (\emph{i.e.}, $\sum_{t=1}^T \max_{\bm{x}} ||[\bm{g}_{t+1}(\bm{x}) - \bm{g}_{t}(\bm{x})]^+||$) are $o(T^{2/3})$.
\citet{cao2018online} consider a bandit feedback setting and, in order to provide sublinear regret and constraint violations, they assume that all the loss functions have uniformly bounded difference (\emph{i.e.}, for each $t$ and $\vx,\vx'\in\cX$, $|f_t(\vx)-f_t(\vx')|\le M$ for some positive constant $M$), and that the drift of the baseline sequence is sublinear.
In other words, the underlying dynamic optimization problems vary \emph{slowly} over time. 
Both \citet{chen2017online} and \citet{cao2018online} need to assume the Slater’s condition.
\citet{yi2020distributed} provide similar results in a distributed online convex optimization setting with adversarial constraints.
They analyze both the case in which the Slater's condition holds, and the case without this assumption.

Others relevant related works, are those studying online learning problems in which the decision maker has to satisfy supply/budget constraints.
In this setting, the decision maker wants to maximize their expected reward without violating a set of $m$ resource constraints.
The process stops at time horizon $T$, or when the total consumption of some resource exceeds its budget.
\citet{badanidiyuru2018bandits} first introduce and solve the \emph{Bandits with Knapsacks} (BwK) framework, in which thay consider bandit feedback, stochastic objective and constraint functions.
Other optimal algorithms for Stochastic BwK were proposed by \citet{agrawal2014bandits, agrawal2019bandits} and by \citet{immorlica2019adversarial}.
The \emph{Adversarial Bandits with Knapsacks} setting was first studied by \citet{immorlica2019adversarial}.
The authors shows that an appropriate baseline is the best fixed distribution over arms.
Achieving no-regret is no longer possible under this baseline and, therefore, they provide no-$\alpha$-regret guarantees for their algorithm.

We remark that in this paper we are able to handle more general constraints than \citet{immorlica2019adversarial}, which can deal only with budget constraints. Moreover, we can compete with a baseline stronger than the static regret used by \citet{sun2017safety}, without needing the strong assumptions required, for instance, by \citet{cao2018online}.

\section{Preliminaries}\label{sec:prel}

The decision maker has a non-empty set of available strategies $\cX$ (this set may be non-convex, integral, and even non-compact).
In each round $t\in[T]$,\footnote{In this work, we denote by $[x]$ the set $\left\{ 1,\ldots, x \right\}$ of the first $x$ natural numbers.} the decision maker first chooses $\vx_t\in\cX$, and the environment selects a reward function $f_t:\cX\to[0,1]$ and a constraint function $g_t:\cX\to[-1,1]^{m}$ conditioned on the past history of play up to time $t-1$ (\emph{i.e.}, the environment chooses $f_t$ and $g_t$ without knowledge of $\vx_t$). Notice that both $f_t$ and $g_t$ need \emph{not} be convex.
The latter specifies a set of $m$ constraints of the form $g_t(\vx) \leq \vzero $, with $g_{t,i}(\vx) \leq 0$ denoting the $i$-th constraint.\footnote{Focusing on the case $g_t(\vx) \leq \vzero $ is w.l.o.g. since any set of constraints can be represented in such a form.}
In the following, we denote as $\mathcal{F}$, respectively $\mathcal{G}$, the set of all the possible $f_t$, respectively $g_t$, functions (\emph{e.g.}, $\mathcal{F}$ and $\mathcal{G}$ may contain all the Lipschitz-continuous functions defined over $\cX$).
At each round $t \in [T]$, the decision maker can condition their decision on prior feedbacks and on the sequence of prior decisions $\vx_1,\ldots,\vx_{t-1}$, but no information about future rewards and constraint functions is available.

\subsection{Strong duality through strategy mixtures} \label{sec:duality}

Next, we define the optimization problem (Problem~\ref{eq:opt lp gen}) which is used to define the baselines against which we compare the performances of the decision maker.
Such a problem involves probabilistic mixtures of strategies in $\cX$, which are crucial in order to recover strong duality.\footnote{The optimal fixed strategy mixture provides an arguably stronger baseline than the optimal fixed strategy. In stochastic settings, this baseline is related to the best dynamic policy. In particular, if we consider the case in which the observed functions are defined as the average of functions $f_t$ and $g_t$ across the $T$ rounds, then the optimal mixture provides the same utility as the best dynamic policy (see \cite{bada2018bandit} for a similar result).}

First, we introduce the set of probability measures on the Borel sets of $\cX$.
We refer to such a set as the set of \emph{strategy mixtures}, denoted as $\Xi$.
We endow $\cX$ with the Lebesgue $\sigma$-algebra, and we assume that all the functions in $\mathcal{F}$ and $\mathcal{G}$ are measurable with respect to every probability measure $\vxi \in \Xi$.
This ensures that the various expectations taken are well-defined, since the functions are assumed to be bounded above, and they are therefore integrable.
In the following, for the ease of presentation and with a slight abuse of notation, whenever we write a $\vxi \in \Xi$ in place of an $\vx \in \cX$, we mean that we are taking the expectation with respect to the probability measure $\vxi$.
For instance, given $f \in \mathcal{F}$ and $g \in \mathcal{G}$, we have that $f(\vxi) = \E_{\vx \sim \vxi} f(\vx)$ and $g(\vxi) = \E_{\vx \sim \vxi} g(\vx)$.

Then, given two functions $f \in \mathcal{F}$ and $g \in \mathcal{G}$, we define the following optimization problem, which chooses the strategy mixture $\vxi \in \Xi$ that maximizes the expected reward encoded by $f$, while guaranteeing that the constraints encoded by $g$ are satisfied in expectation. 
\begin{equation}\label{eq:opt lp gen}
	\tag{$\cP_{f,g}$}
	\OPT_{f,g}\defeq\mleft\{\begin{array}{lll}
		\displaystyle
		\sup_{\vxi \in \Xi} & f(\vxi) & \text{s.t.} \\
		&  g(\vxi)\le 0.
	\end{array}\mright.
\end{equation}

We denote by $d_g \in [-1,1]$ the largest possible value for which there exists a strategy mixture $\vxi \in \Xi$ satisfying the constraints $g(\vxi)\le 0$ by a margin of at least $d_g$.
Formally, 
\begin{equation}\label{eq:dg}
d_g \defeq  \sup_{\vxi \in \Xi} \min_{i \in [m]} - g_i(\vxi).
\end{equation}
In order to ensure that $\OPT_{f,g}$ is always well defined, we assume that it is always the case that $d_g \geq 0$.
Notice that, if $d_g > 0$, then Problem~\ref{eq:opt lp gen} satisfies Slater's condition.

In the following, we prove some auxiliary results relating to Problem~\ref{eq:opt lp gen} that will be useful in the rest of the paper.
First, we introduce a Lagrangian relaxation of the problem.
\begin{definition}[Lagrangian Function]\label{def:lagrangian}
	Given two arbitrary functions $f \in \mathcal{F}$ and $g \in \mathcal{G}$, the \emph{Lagrangian function} $\cL_{f,g}:\Xi\times\R_{\ge 0}^m \to\R$ of Problem~\ref{eq:opt lp gen} is defined as
	\[
	\cL_{f,g}(\vxi,\vlambda)\defeq  f(\vxi) - \langle\vlambda, g(\vxi)\rangle.
	\]
\end{definition}
If Problem~\ref{eq:opt lp gen} satisfies Slater's condition, then Theorem~1 of Chapter~8.3 in~\citep{luenberger1997optimization} readily gives us that strong duality holds even if $f$ and $ g$ are arbitrary non-convex functions.
Formally:
\begin{corollary}%
	\label{thm:duality slater}
	Given $f \in \mathcal{F}$ and $g \in \mathcal{G}$ such that $d_g>0$, it holds
	\[
	\sup_{\vxi  \in \Xi } \, \inf_{\vlambda \in \R_{\ge 0}^m} \, \cL_{f,g}(\vxi,\vlambda)= \inf_{\vlambda \in \R_{\ge 0}^m} \, \sup_{\vxi  \in \Xi } \,  \cL_{f,g}(\vxi,\vlambda)=\textnormal{$\OPT_{f,g}$}.
	\]
\end{corollary}

Next, we show that, if $d_g>0$, then strong duality holds even when we restrict the admissible dual vectors $\vlambda \in \R_{\ge 0}^m$ to the set $\cD_{d_g}$, where, for any $q \in \mathbb{R}_{>0}$, we let
\(
	\cD_q \defeq\mleft\{\vlambda\in\R_{\geq 0}^m :\|\vlambda\|_1 \leq 1/q\mright\} 
\)
(omitted proofs can be found in Appendix \ref{sec:omitted proofs}).
\begin{restatable}{theorem}{theoremDuality}\label{thm:duality_restricted}
	Given $f \in \mathcal{F}$ and $g \in \mathcal{G}$ such that $d_g>0$, it holds
	\[
		\sup_{\vxi  \in \Xi } \inf_{\vlambda\in \cD_{d_g}} \, \cL_{f,g}(\vxi,\vlambda)= \inf_{\vlambda\in\cD_{d_g}} \sup_{\vxi  \in \Xi} \, \cL_{f,g}(\vxi,\vlambda)=\textnormal{$\OPT_{f,g}$}.
	\]
\end{restatable}

\subsection{Stochastic vs. adversarial: baselines and feasibility}

We consider several settings, differing in how functions $f_t$ and $g_t$ are selected, either \emph{stochastically} or \emph{adversarially}.
We say that functions $f_t$ (respectively $g_t$) are selected stochastically, when they are independently drawn according to a given probability measure $\mu_{\cF}$ over $\mathcal{F}$ (respectively  $\mu_{\cG}$ over $\mathcal{G}$).
Instead, we say that functions $f_t$ (respectively $g_t$) are selected adversarially if each $f_t$ (respectively $g_t$) is chosen by an adversary based on the sequence of prior decisions, namely $\vx_1, \ldots, \vx_{t-1}$.

Consistently with previous work (see, \emph{e.g.},~\citep{mannor2009online}), we compare the performance of the decision maker (in terms of reward cumulated over the $T$ rounds) against the baseline $T \, \OPT_{\bar f, \bar g}$ (as defined by Problem~\hyperref[eq:opt lp gen]{$\cP_{\bar f,\bar g}$}), where $\bar f$ and $\bar g$ are suitably-defined functions.
In particular:
\begin{itemize}
	\item When functions $f_t$, respectively $g_t$, are selected stochastically, then we define function $\bar f$, respectively $\bar g$, so that $\bar f(\vx)\defeq \mathbb{E}_{f \sim \mu_\cF} \mleft[f(\vx)\mright] $, respectively $  \bar  g(\vx)\defeq \mathbb{E}_{g \sim \mu_\cG} \mleft[g(\vx)\mright] $.
	\item When functions $f_t$, respectively $g_t$, are selected adversarially, then we define function $\bar f$, respectively $\bar g$, so that $\bar f(\vx)\defeq \frac{1}{T}\sum_{t=1}^T f_t(\vx) $, respectively $\bar  g(\vx)\defeq \frac{1}{T}\sum_{t=1}^T g_t(\vx)$.
\end{itemize}
Intuitively, in the stochastic case, the baseline is instantiated with an expectation of functions taken with respect to the probability measure $\mu_\cF$ (respectively $\mu_\cG$).
Instead, in the adversarial case, the baseline uses the average of functions $f_t$ (respectively $g_t$) observed over the $T$ rounds.

Let us remark that, when the set $\cX$ is compact convex and functions $f_t$ and $g_t$ are convex, then Problem~\hyperref[eq:opt lp gen]{$\cP_{\bar f,\bar g}$} defining our baselines can be equivalently re-written by using strategies $\vx \in \cX$ rather than strategy mixtures $\vxi \in \Xi$, since there always exists an optimal solution to Problem~\hyperref[eq:opt lp gen]{$\cP_{\bar f,\bar g}$} that places all the probability mass on a single strategy.

Our goal is to design online algorithms for the decision maker that output a sequence of decisions $\vx_1,\ldots,\vx_T$ such that both the \emph{cumulative regret} with respect to the performance of the baseline, defined as $R^T \coloneqq T \, \OPT_{\bar f, \bar g} - \sum_{t=1}^T f_t(\vx_t)$, and the \emph{cumulative constraints violation}, defined as $V^T \coloneqq \max_{i \in [m]} \sum_{t =1}^T g_{t,i}(\vx_t)$, grow sublinearly in the number of rounds $T$.

In conclusion, we introduce a problem-specific parameter that is strictly related to the feasibility of Problem~\hyperref[eq:opt lp gen]{$\cP_{\bar f,\bar g}$}.
We call it the \emph{feasibility parameter} $\rho \in \R$, which is formally defined as follows:
	\begin{itemize}
		\item When functions $g_t$ are selected stochastically, $\rho\defeq  \sup_{\vxi \in \Xi} \min_{i \in [m]} - \bar g_i(\vxi)$.
		\item When functions $g_t$ are selected adversarially, $  \rho \defeq  \sup_{\vxi \in \Xi} \min_{t \in [T]} \min_{i \in [m]} -  g_{t,i}(\vxi)$.
	\end{itemize}
Intuitively, in the stochastic case, $\rho$ is equal to $d_{\bar g}$, while in the adversarial case it is computed similarly, but considering the worst case with respect to the functions $g_t$ observed at each round $t$. 
Notice that, when $\rho > 0$, Slater's condition is satisfied for Problem~\hyperref[eq:opt lp gen]{$\cP_{\bar f,\bar g}$}.

In the following, we denote by $\vxi^* \in \Xi$ a strategy mixture that is optimal for Problem~\hyperref[eq:opt lp gen]{$\cP_{\bar f,\bar g}$}.
Moreover, we always assume that functions $f_t$ and $g_t$ are such that Problem~\hyperref[eq:opt lp gen]{$\cP_{\bar f,\bar g}$} is feasible, and we let $\vxi^\circ \in \Xi$ be the \emph{feasible} strategy mixture that is optimal for the problem defining $\rho$.\footnote{Notice that $\vxi^*$ and $\vxi^\circ$ may \emph{not} be well defined in all the cases in which the problem that defines them does \emph{not} admit a maximum. Nevertheless, in such cases, we assume that $\vxi^*$ (or $\vxi^\circ$) is a strategy mixture arbitrarily ``close'' to the supremum, so that all of our results continue to hold up to negligible additive approximations that are dominated by other approximation factors, and we can safely ignore them for the ease of exposition.}

\subsection{Regret minimizers}

A \emph{regret minimizer} (RM) for a set $\cW$ is an abstract model for a decision maker that repeatedly interacts with a black-box environment. 
At each $t$, a RM performs two operations: (i) \nextelement, which outputs an element $\vw_t \in \cW$; and (ii) \observeutility, which updates the internal state of the RM using the feedback received from the environment.
This is defined in terms of a utility function $u_t:\cW\to [a,b]$ having range $[a,b]\subseteq \mathbb{R}$, with $u_t$ possibly depending adversarially on the sequence of outputs $\vw_1,\ldots,\vw_{t-1}$.
The objective of the RM is to output a sequence $\vw_1,\ldots,\vw_T$ of points in $\cW$ so that its \emph{cumulative regret}, defined as $ \sup_{\vw \in\cW}\sum_{t=1}^T\mleft(u_t(\vw)-u_{t}(\vw_t)\mright)$,
grows asymptotically sublinearly in $T$.
See~\citep{cesa2006prediction} for a review of the various RMs available in the literature.

For the ease of presentation, we introduce the concept of \emph{regret minimizer constructor}, which is a procedure, say $\ini(\cW,[a,b],\fail)$, that builds a RM on the basis of the three parameters given as input.
In particular, the procedure returns a RM instantiated for the set $\cW$, working with utility functions having range $[a,b]$, and such that its cumulative regret is guaranteed to grow sublinearly in the time horizon $T$ with probability at least $1-\fail$.

\section{A unifying meta-algorithm}\label{sec:alg}

In this section, we present our meta-algorithm.
Its core idea is to instantiate suitable pairs of RMs, where one is working in the domain $\cX$ of primal variables and the other in a suitable subset of the domain $\mathbb{R}^m_+$ of dual variables.
At each round $t \in [T]$, the algorithm makes the RMs ``play'' against each other in a  \emph{Lagrangian game}, where the utility functions observed by them are related to the Lagrangian function $\cL_{f_t,g_t}(\vx,\vlambda)$ of Problem \hyperref[eq:opt lp gen]{$\cP_{f_t,g_t}$}.\footnote{The idea of having pairs of primal, dual RMs playing a {Lagrangian game} was originally introduced by~\citet{immorlica2019adversarial}, restricted to the case of knapsack constraints.}

Algorithm~\ref{alg:meta alg known} provides the pseudo-code of the meta-algorithm, which takes as input: the total number of rounds $T$, a failure probability $\delta \in (0,1)$ such that the guarantees provided by the algorithm hold with probability at least $1-\delta$, and a lower bound $\hat \rho \geq 0$ on the value of the feasibility parameter $\rho$.

\paragraph{Algorithm description.}
The algorithm works in two phases.
In the first one, called \emph{play phase}, the algorithm builds a primal RM, called $\cRpuno$, working in the primal domain $\cX$ and a dual RM, called $\cRduno$, operating on the subset $\cD_{\tilde \rho}$ of the dual domain $\mathbb{R}_+^m$, where $\tilde\rho$ is set in Line~\ref{line:rho tilde}.
The algorithm makes the two RMs playing against each other (see the call $\textsc{LagrangianGame}(\cRpuno,\cRduno,1)$) until either the cumulative violation $V^t$ incurred by the algorithm exceeds a given threshold (see Line~\ref{line:while}, where $M_{\tilde \rho}$ is defined in Equation~\eqref{def:m}) or round $T$ is reached.
Then, in the second phase, called \emph{recovery phase}, the algorithm constructs a new pair of primal, dual RMs, with the latter working on the $(m-1)$-dimensional simplex $\Delta_m$.
The recovery phase uses the remaining rounds to make these new RMs play against each other, with the primal RM observing modified utility functions that do \emph{not} account for functions $f_t$ (see the call $\textsc{LagrangianGame}(\cRpdue,\cRddue,0)$).
Intuitively, this is needed in order to ensure that the algorithm plays strategies $\vx_t$ that satisfy the constraints, thus balancing out the cumulative constraint violation accumulated in the first phase.
The pseudo-code describing one ``play'' between two RMs, called $\cRp$ and $\cRd$, is defined by the sub-procedure $\textsc{LagrangianGame}(\cRp,\cRd,v)$ in Algorithm~\ref{alg:lagrangian}.
The additional parameter $v \in \{0,1\}$ is used to control the feedback fed into the primal RM $\cRp$; specifically, if $v=1$, then $\cRp$ observes a utility function that also accounts for $f_t$ (play phase), otherwise, if $v=0$, the observed utility function only accounts for the term depending on $g_t$ (recovery phase).

\begin{figure}[!htp]
	\begin{minipage}{0.49\textwidth}
		\begin{algorithm}[H]
			\caption{\textsc{Meta-Algorithm}$(T,\delta,\hat \rho)$}
			\label{alg:meta alg known}
			\begin{algorithmic}[1]
				\State $\tilde \rho \gets \max \left\{ \hat \rho/2,T^{-1/4} \right\}, \eta \gets \delta / 3, t\gets 1$\label{line:rho tilde}
				\LeftComment \textcolor{gray}{{Phase I: Play}}
				\State $\cRpuno \gets \inip\left( \cX, \big[ -1 / \tilde{\rho},1+1 / \tilde{\rho} \big],\eta\right)$
				\State $\cRduno \gets \inid\left( \cD_{\tilde \rho}, \big[ -1 / \tilde{\rho},1 / \tilde{\rho} \big],0\right)$
				\While{$V^t \le (T-t) \tilde \rho+M_{\tilde \rho}-1 \wedge t\le T$}\label{line:while}
				\State $\vx_t \gets  \textsc{LagrangianGame}(\cRpuno,\cRduno,1 )$
				\State $t\gets t+1$
				\EndWhile
				\State $T_1 \gets t-1$
				\LeftComment \textcolor{gray}{{Phase II: Recovery}}
				\State $\cRpdue \gets \inip\left( \cX, [ -1,1 ],\eta\right)$
				\State $\cRddue \gets \inid\left( \Delta_m, [-1,1],0 \right)$
				\While{$t \leq  T$}
				\State $\vx_t \gets \textsc{LagrangianGame}(\cRpdue,\cRddue ,0)$
				\State $t\gets t+1$
				\EndWhile
			\end{algorithmic}
		\end{algorithm}
	\end{minipage}
	\hfil
	\begin{minipage}{0.5\textwidth}
		\vspace{-2.25\baselineskip}
		\begin{algorithm}[H]
			\caption{\textsc{LagrangianGame}$(\cRp,\cRd,v)$}
			\label{alg:lagrangian}
			\begin{algorithmic}[1]
				\State $\vx_t\gets \cRp.\nextelement$ %
				\State $\vlambda_t\gets\cRd.\nextelement$ %
				\State {\setlength{\tabcolsep}{0pt}
					\newlength\q
					\setlength\q{\dimexpr .7\textwidth}\noindent\begin{tabular}{p{\q}r}
						Play $\vx_t$ and get $f_t$ and $g_t$ \hfil& \Comment{\textcolor{gray}{Full f.}} \\ Play $\vx_t$ and get $f_t(\vx_t)$ and $g_t(\vx_t)$\hfil &\Comment{\textcolor{gray}{Bandit f.}}
				\end{tabular}}
				\LeftComment \textcolor{gray}{Primal RM update}
				\State {\setlength{\tabcolsep}{0pt}
					\setlength\q{\dimexpr .7\textwidth}\noindent\begin{tabular}{p{\q}r}
						Let $\lossp: \vx \mapsto v f_t(\vx) - \langle\vlambda_t,  g_t(\vx)\rangle$ \hfil &\Comment{\textcolor{gray}{Full f.}} \\
						$\lossp(\vx_t) \gets v f_t(\vx_t) - \langle\vlambda_t,  g_t(\vx_t)\rangle$ \hfil &\Comment{\textcolor{gray}{Bandit f.}}
				\end{tabular}}
				\State {\setlength{\tabcolsep}{0pt}
					\setlength\q{\dimexpr .7\textwidth}\noindent\begin{tabular}{p{\q}r}
						$\cRp.\observeutility[\lossp]$  \hfil &\Comment{\textcolor{gray}{Full f.}} \\ $\cRp.\observeutility[\lossp(\vx_t)]$ \hfil &\Comment{\textcolor{gray}{Bandit f.}}
				\end{tabular}}
				\LeftComment \textcolor{gray}{Dual RM update}
				\State Let $\lossd:\vlambda\mapsto -\langle\vlambda, g_t(\vx)\rangle$
				\State $\cRd.\observeutility[\lossd]$ %
			\end{algorithmic}
		\end{algorithm}
	\end{minipage}
\end{figure}

\paragraph{Regret minimizer constructors.}
Algorithm~\ref{alg:meta alg known} also needs access to two suitably-defined regret minimizer constructors, namely $\inip(\cW,[a,b],\fail)$ and $\inid(\cW,[a,b],\fail)$, where the former is used to build RMs working in the primal domain and the latter for those operating on the dual one.
Their actual implementation depends on the specific problem at hand.
In the following, we let $\cump_{t,\eta}$ be the regret upper bound (on $t \in [T]$ rounds) for primal RMs $\cRp$ dealing with utility functions having range $[0,1]$, as returned by the call $\inip(\cX,[0,1],\fail)$.
Notice that, when the range is $[a,b]$, the same RM can be adopted by first normalizing utility values, so that the resulting regret upper bound is $(b-a)\cump_{t,\eta}$.
As for dual RMs $\cRd$, we let $\cumd_t$ be the regret upper bound (on round $t \in [T]$) provided by the RM defined for the set $\Delta_m$, while $\cumd_t/\tilde \rho$ is the upper bound for the dual RM instantiated on the set $\cD_{\tilde \rho}$.
Notice that, since dual RMs always have full feedback, we can safely assume that the regret bounds $\cumd_t$ hold deterministically.
We also assume that RMs provide bounds that increase with the number of rounds, \emph{i.e.}, such that $\cump_{t,\eta} \le \cump_{t',\eta}$ and $\cumd_t \le \cumd_{t'}$ for all $t \le t'$.

\paragraph{How to construct RMs.}
$\inid$ can be implemented by using \emph{online mirror descent} (OMD) with domain $\Delta_m$ (or $\cD_1$) and a negative entropy regularizer.
Since the utility function $\lossd$ is linear in $\vlambda$, we get a regret bound for the primal RM of $\cumd_{T}=O(\sqrt{T \log(m)})$ (see, \emph{e.g.}, \cite{beck2003mirror,nemirovskij1983problem}). 
The design of $\inip$ depends on the structure of $\cX$ and functions $f_t$ and $g_t$.
For instance, in convex settings with full feedback we can employ OMD~\citep{hazan2019Introduction}, while with bandit feedback we can use~\citep{bubeck2017Kernel}.
Finally, for non-convex functions we can employ, \emph{e.g.}, the RMs in~\citep{suggala2020Online}.
All these RMs guarantee $\tilde O(\sqrt T)$ regret.

\paragraph{How to get away with no knowledge of $\rho$.}
In \Cref{sec:unknown}, we show that a lower bound $\hat \rho$ is \emph{not} necessary when functions $g_t$ are selected stochastically.
Indeed, it is sufficient to add a preliminary phase to Algorithm~\ref{alg:meta alg known}, which is used to infer a suitable lower bound on $\rho$ from experience.
In order to do this, only $\sqrt{T}$ rounds are needed, so that the bounds of Algorithm~\ref{alg:meta alg known} are \emph{not} compromised.
When functions $g_t$ are chosen adversarially, it is easy to see that it is impossible to compute a lower bound on the feasibility parameter $\rho$ by only using the first rounds. 
For instance, think of a setting in which $\rho$ is very large when only considering the first rounds, while it becomes small during later rounds.

\begin{remark}[Dependence on the lower bound $\hat \rho$]
	Algorithm~\ref{alg:meta alg known} can take as input any $\hat \rho \ge 0$. However, since our regret bounds include a factor $1 / {\tilde \rho}$, by choosing the trivial lower bound $\hat \rho=0$ we incur in a regret of $\tilde O(\sqrt{T}/ \tilde \rho)= \tilde O( T^{3/4})$. In order to obtain optimal bounds, we would like to have $\tilde \rho = \Omega(\rho)$.
\end{remark}

\begin{remark}[Dependence on the feasibility parameter $\rho$]
	We choose to include the dependence on the feasibility parameter $\rho$ in the order of convergence of the algorithm. As customary, the goal is devising bounds in the form $\text{\normalfont poly(instance)}\cdot h(T)$, where the first term is a polynomial function of the parameters defining the problem instance, and $h(T)=o(T)$. Therefore, we cannot include a factor $1/\rho$ in the regret bounds if $\rho$ can be arbitrarily small. Even from a practical standpoint, when $\rho$ is too small a $1/\rho$ regret bound is too large to be significant.
	For those reasons, we set $\tilde \rho$ in Algorithm~\ref{alg:meta alg known} to be the maximum between the feasibility parameter lower bound $\hat \rho$ and $T^{-1/4}$.
	The value $T^{-1/4}$ has been chosen so as to minimize the maximum between the cumulative regret and the cumulative constraint violation when the lower bound on the feasibility parameter $\hat \rho$ is too small. %
\end{remark}

\section{Analysis with stochastic constraints and adversarial rewards} \label{sec:advStoc}

We start by analyzing the performance of our meta-algorithm (Algorithm~\ref{alg:meta alg known}) when the reward and constraint functions are selected stochastically and adversarially, respectively.

Given $t \in [T]$ and $\eta \in (0,1)$, we let $\err_{t,\eta}\coloneqq \sqrt{8t \log (18mt^2 / \eta)}$ be the value bounding differences between expectations and empirical means of constraint functions, obtained by applying the Azuma-Hoeffding inequality, and holding with probability at least $1-\eta$.
Given $\gamma \in (0,1)$, we also let
\begin{equation}\label{def:m}
		M_\gamma \coloneqq \frac{2}{\gamma}\sqrt{T}+  \left( 2+ \frac{3}{\gamma} \right) \err_{t,\eta} + \left( 1+\frac{2}{\gamma} \right)\cump_{t,\eta} +\frac{1}{\gamma} \cumd_t,
\end{equation}
which is a recurring term related to the maximum violation that Algorithm~\ref{alg:meta alg known} accepts in play phase.

First, we introduce a useful event $\CE$ that encompasses all the cases in which Algorithm~\ref{alg:meta alg known} successfully terminates.
Then, Lemma~\ref{lm:ce} shows that such an event holds with probability at least $1-\delta$.
In particular, $\CE$ holds when the regret bounds of $\cRpuno$ and $\cRpdue$ hold, and, additionally, the differences between expectations and empirical means of constraint functions are bounded as desired.
\begin{definition}\label{def:event_e}
	We denote with $\CE$ the event in which Algorithm~\ref{alg:meta alg known} satisfies the following conditions (recall that $\eta = \delta / 3$): (i) the regret incurred by $\cRpuno$ after $T_1$ rounds is upper bounded by $\cump_{T_1,\eta}$; (ii) the regret cumulated by $\cRpdue$ after the remaining $T - T_1$ rounds is upper bounded by $\cump_{T- T_1,\eta}$; and (iii) for every pair of rounds $t, t' \in[T]: t \le t'$ and resource $i \in [m]$ it holds:
	\begin{OneLiners}
		\item $\left| \sum_{\tau = t}^{t'}  g_{\tau,i}(\vx_\tau)-\sum_{\tau = t}^{t'} \bar{g}_i (\vx_{\tau}) \right| \le \err_{t'-t,\eta}$,
		\item $\left| \sum_{\tau = t}^{t'}  \vlambda_\tau \, g_{\tau,i} (\vx_\tau)-\sum_{\tau = t}^{t'} \vlambda_\tau \, \bar{g}_i (\vx_\tau) \right|\le \err_{t'-t,\eta} \max_{\tau \in [T]: t\leq \tau \leq t'} ||\vlambda_\tau||_1 $,
		\item  $\left| \sum_{\tau = t}^{t'}  g_{\tau,i} (\vxi)-\sum_{\tau = t}^{t'} \bar{g}_i (\vxi) \right| \le \err_{t'-t,\eta}$ for $\vxi \in \{ \vxi^*, \vxi^\circ \}$,
		\item $\left| \sum_{\tau = t}^{t'}  \vlambda_\tau  g_{\tau,i} (\vxi)-\sum_{\tau = t}^{t'} \vlambda_\tau  \bar{g}_i (\vxi) \right|\le \err_{t'-t,\eta} \max_{\tau \in [T]: t\leq \tau \leq t'} ||\vlambda_\tau||_1$ for $\vxi \in \{ \vxi^*, \vxi^\circ \}$.
	\end{OneLiners}
\end{definition}
\begin{restatable}{lemma}{lemmaAzuma} \label{lm:ce}
	After running Algorithm~\ref{alg:meta alg known}, the event $\CE$ holds with probability at least $1-\delta$.
\end{restatable}

Next, we lower bound the cumulative reward obtained by Algorithm~\ref{alg:meta alg known} during the play phase.
Intuitively, we show that, if the cumulative constraints violation is large, then the decisions $\vx_t$ in the first $T_1$ rounds provide a per-round reward much higher than that achievable by $\vxi^*$. 
This allows us to employ the following recovery phase to decrease constraints violation cumulated in the play phase, while also ensuring that the cumulative regret stays low at the end of the algorithm.
Formally:
\begin{restatable}{lemma}{lemmaSmallRegret}\label{lm:smallReg}
	If event $\CE$ holds,
	then after round $T_1$ of Algorithm~\ref{alg:meta alg known} the following inequality holds:
	\textnormal{
	\(
		\sum_{t =1}^{T_1} f_t(\vx_t) \ge \sum_{t =1}^{T_1}  f_t(\vxi^*) + (T-T_1) -\frac{1}{\tilde \rho}\err_{T_1,\eta}-  \left( 1+\frac{2}{\tilde \rho} \right)\cump_{T_1,\eta}-\frac{1}{\tilde{\rho}}\cumd_{T_1}.
	\)}
\end{restatable}

In the recovery phase, the only goal of Algorithm~\ref{alg:meta alg known} is to decrease constraints violation.
In the following Lemma~\ref{lm:smallViolation}, we show that, at each round of the recovery phase, the algorithm is ``close'' to satisfying (in expectation) all the constraints by at least $\rho$.
Formally:
\begin{restatable}{lemma}{lemmaSmallViolation}\label{lm:smallViolation}
	If event $\CE$ holds, then after Algorithm~\ref{alg:meta alg known} halts, the following holds for every $i\in [m]$:
	\textnormal{
	\(
		\sum_{t = T_1+1}^T g_{t,i}(\vx_t)\le -(T-T_1) \rho +2\cump_{T-T_1, \eta} + \cumd_{T-T_1}+ \err_{T-T_1,\eta}.
	\)}
\end{restatable}

Now, we are ready to present the two main results of this section.
First, we provide a bound on the cumulative regret and constraints violation when the lower bound $\hat \rho$ is sufficiently large.
\begin{condition} \label{ass:strictly stoc}
	It holds that $\hat  \rho \ge 2 T^{-1/4}$.
\end{condition}
Notice that, under Condition~\ref{ass:strictly stoc}, $\tilde \rho = \hat \rho / 2$.
This gives us the following result:
\begin{restatable}{theorem}{theoremAdvStocSlater}\label{thm:withSlater}
	Suppose that functions $f_t$ and $g_t$ are selected adversarially and stochastically, respectively.
	If Condition~\ref{ass:strictly stoc} is satisfied, then, with probability at least $1-\delta $, Algorithm~\ref{alg:meta alg known} provides \textnormal{$R^T \leq \frac{1}{\tilde  \rho} \err_{T,\eta}+\left( 1+\frac{2}{\tilde  \rho} \right)\cump_{T,\eta}+\frac{1}{\tilde \rho}\cumd_{T}$} and \textnormal{$V^T \leq M_{\tilde \rho}+2\cump_{T,\eta} + \cumd_{T}+  \err_{T,\eta} $}.
\end{restatable}

Finally, we also prove that even if Condition~\ref{ass:strictly stoc} is \emph{not} satisfied, \emph{i.e.}, the lower bound $\hat \rho$ is \emph{not} sufficiently large, the following holds:
\begin{restatable}{theorem}{theoremAdvStocWithout}\label{thm:withoutSlater}
	Suppose that functions $f_t$ and $g_t$ are selected adversarially and stochastically, respectively.
	Algorithm~\ref{alg:meta alg known} guarantees that the following bounds hold with probability at least $1-\delta$: \textnormal{$R_T \leq T^{1/4}\err_{T,\eta}+\left( 1+2T^{1/4} \right)\cump_{T,\eta}+T^{1/4}\cumd_{T}$} and \textnormal{$V_T \leq T^{3/4}+ M_{T^{-1/4}}+ 2 \cump_{T,\eta} + \cumd_T+  \err_{T,\eta} $}.
\end{restatable}

\begin{remark}
	Notice that, by using primal and dual RMs whose regret bounds are of the order of $\tilde O(\sqrt{T})$, Theorem~\ref{thm:withSlater} allows us to recover $\tilde O(\sqrt{T}/\hat \rho)$ regret and $\tilde O(\sqrt{T}/\hat \rho)$ constraints violation for the case in which Condition~\ref{ass:strictly stoc} holds. Theorem~\ref{thm:withoutSlater} still provides $\tilde O(T^{3/4})$ regret and constraints violation when the condition is {not} met, which is necessary the case when $\rho = 0$.
\end{remark}

\section{Analysis with stochastic constraints and stochastic rewards} \label{sec:stoc}

In this section, we focus on the case in which both reward and constraint functions are selected stochastically.
In this setting, we are able to show that Algorithm~\ref{alg:meta alg known} never enters the recovery phase. 
As we argue in Section \ref{sec:app}, this is an important property for budget-management applications, since it is related to the round in which the budget is fully depleted.

In order to prove our result, we extend the event $\CE$ to capture also the Azuma-Hoeffding bounds for the reward functions, which are stochastic in this setting.\footnote{Accounting for the martingale difference sequences $f_t(\vx_t)-\bar f(\vx_t)$ and $f_t(\vxi^*)-\bar f(\vxi^*)$.}
The core idea that we exploit to prove our result is that we can think of the two RMs as if they are playing a stochastic repeated zero-sum game, which is the repeated Lagrangian game whose functions are sampled according to the probability measures $\mu_{\mathcal{F}}$ and $\mu_{\mathcal{G}}$.
By Theorem~\ref{thm:duality_restricted}, strong duality holds, and the game has an equilibrium.
Hence, it is possible to show that the per-round utility of the primal RM is close to the value of the game, which is $\OPT_{\bar f,\bar g}$.
At the same time, it is possible to show that, if the cumulative constraints violation becomes large during the play phase (and, thus, $T_1 < T$), then the per-round utility of the primal RM is below~$\OPT_{\bar f,\bar g}$, reaching a contradiction that proves the following theorem.
\begin{restatable}{theorem}{theoremStoc} \label{thm:sto}
	Suppose that functions $f_t$ and $g_t$ are selected stochastically.
	With probability at least $1-\delta$, Algorithm~\ref{alg:meta alg known} never enters the recovery phase, namely $T_1 = T$.
\end{restatable}
Notice that regret bounds analogous to the one in Theorems~\ref{thm:withSlater}~and~\ref{thm:withoutSlater} also hold in the case in which both reward and constraint functions are selected stochastically.

\section{Analysis with adversarial constraints} \label{sec:adv}

In this section, we study settings in which the constraint functions $g_{t}$ are selected adversarially.
As shown by~\citet{mannor2009online}, it is impossible to obtain sublinear cumulative regret and constraints violation when using our baseline, \emph{i.e.}, the best fixed strategy mixture $\vxi^*$ satisfying (in expectation) the long-tern constraints.
However, 
we show that it is possible to achieve a $\rho/(1+\rho)$ fraction of the cumulative reward obtained by always playing $\vxi^*$, while guaranteeing sublinear constraints violation.
The dependence of the approximation factor on the feasibility parameter $\rho$ is similar to the dependence on the per-round budget in problems with budget constraints (see the related works in \Cref{sec:related} for more details).
Moreover, as we discuss later in Section~\ref{sec:app}, when restricted to the case of budget constraints and adversarial reward/cost functions, our approximation factor matches the state-of-the-art bounds provided by~\citet{castiglioni2022Online}.

As a first step to prove our result, we provide a lower bound on the cumulative reward of the primal RM during the play phase.
We show that it achieves at least a $\rho/(1+\rho)$ fraction of the value obtained by the optimal solution in the first $T_1$ rounds.
Moreover, the algorithm provides an additional utility compensating for the last rounds in which the algorithm only focuses in satisfying the constraints.
Finally, we show that, in the recovery phase, the constraints are satisfied by at least $\rho$ at each round, up to a term related to the regret of $\cRpdue$ and $\cRddue$, proving the following theorem.
\begin{restatable}{theorem}{theoremAdv}\label{thm:advAdv}
	Suppose that functions $f_t$ and $g_t$ are selected adversarially.
	If Condition~\ref{ass:strictly stoc} is satisfied, then, with probability at least $1-\frac{2}{3}\delta$, Algorithm \ref{alg:meta alg known} guarantees that the following holds:
	\textnormal{\(
	\sum_{t=1}^{T} f_t(\vx_t) \ge \frac{\rho}{1+\rho} \sum_{t=1}^{T} \OPT_{\bar f, \bar g}   -\mleft(1+\frac{2}{\tilde \rho}\mright)\cump_{T,\eta}- \frac{1}{\tilde \rho} \cumd_T
	\)}
	and \textnormal{$V^T \leq M_{\tilde \rho}+2 \cump_{T,\eta} + \cumd_T$}.
\end{restatable}

A similar result can be also derived for the case of stochastic rewards and adversarial constraints.
\begin{restatable}{corollary}{corollaryAdv} \label{cor:corAdv}
	Suppose functions $f_t$ and $g_t$ are selected stochastically and adversarially, respectively.
	If Condition~\ref{ass:strictly stoc} is satisfied, then, with probability at least $1-\delta$, Algorithm \ref{alg:meta alg known} provides 
	\textnormal{\(
	\sum_{t=1}^{T}  f_t(\vx_t) \ge \frac{\rho}{1+\rho} \sum_{t=1}^{T} \OPT_{\bar f, \bar g}  -\mleft(1+\frac{2}{\tilde \rho}\mright)\cump_{T,\eta}- \frac{1}{\tilde \rho} \cumd_T - 2 \err_{T,\eta}
	\)}
    and $V^T \leq M_{\tilde \rho}+2 \cump_{T,\eta} + \cumd_T+ \err_{T,\eta}$.
\end{restatable}

\begin{remark}
	By using primal and dual RMs whose regret bounds are of the order of $\tilde O(\sqrt{T})$, Theorem~\ref{thm:advAdv} and Corollary~\ref{cor:corAdv} allows us to recover $\sum_{t=1}^{T} f_t(\vx_t) \ge \frac{\rho}{1+\rho} \sum_{t=1}^{T} \OPT_{\bar f, \bar g}   -\tilde O(\sqrt{T}/\hat \rho)$,  and $\tilde O(\sqrt{T}/\hat \rho)$ constraints violation for the case in which Condition~\ref{ass:strictly stoc} holds.
\end{remark}

\section{How to get away with no knowledge about the feasibility parameter}\label{sec:unknown}

We show how to extend Algorithm~\ref{alg:meta alg known} in order to deal with settings in which a lower bound on the feasibility parameter $\rho$ is \emph{not} known.
Indeed, we propose an algorithm (Algorithm~\ref{alg:meta alg}) that directly runs Algorithm~\ref{alg:meta alg known}, by first devoting a given number $T_0 < T$ of rounds to inferring a suitable lower bound $\hat \rho$ on the feasibility parameter $\rho$.
Ideally, we would like to have $\hat \rho= \Omega(\rho)$, so that, we recover bounds of the order $\tilde O(\sqrt{T}/\rho)$.
In particular, we show that we can run Algorithm~\ref{alg:meta alg} 
with $T_0 = T^{1/2}$ in order to recover an approximation of $\rho$ that has an additive approximation error of the order $T^{1/4}$.
This is sufficient to get $\hat \rho= \Omega(\rho)$, since a good approximation of $\rho$ is only needed  when $\rho\ge T^{1/4}$.\footnote{Notice that Algorithm~\ref{alg:meta alg} is \emph{not} an explore and exploit algorithm. Indeed, it uses the exploration rounds only to have a rough estimate of $\rho$.}

Let us remark that our approach only works when constraints functions $g_t$ are selected stochastically.
When these are chosen adversarially, it is easy to see that it is impossible to compute a lower bound on the feasibility parameter $\rho$ by only using the first rounds. 
For instance, think of a setting in which $\rho$ is very large by only considering the first rounds, while it becomes small during later rounds.

\begin{algorithm}[!htb]
	\caption{\textsc{Meta-Algorithm}$(T,T_0,\delta)$}
	\label{alg:meta alg}
	\begin{algorithmic}[1]
		\State $\cRp \gets \inip\left( \cX, \big[ -1 ,1\big],\delta \right)$
		\State $\cRd \gets \inid\left( \Delta_m, [-1,1],0 \right)$
		\State $t \gets 1$
		\While{$t \leq T_0$:}
		\State $\vx_t \gets \textsc{LagrangianGame}(\cRp,\cRd,0)$ 
		\State $t \gets t+1$
		\EndWhile
		\State $\hat \rho \gets - \frac{1}{T_0} \left( \max_{i \in [m]}\sum_{t=1}^{T_0}g_{t,i}(\vx_t) + \err_{ T_0,\delta} \right)$
		\State Run Algorithm~\ref{alg:meta alg known} with $T- T_0$, $\delta$, and $\hat \rho$ as inputs
	\end{algorithmic}
\end{algorithm}

In order to exploit the guarantees of Algorithm~\ref{alg:meta alg known} presented in the previous sections, it is enough to show that, after the first $T_0$ rounds of Algorithm~\ref{alg:meta alg}, $\hat \rho\le \rho$ holds with high probability.
\begin{restatable}{lemma}{lemmaGoodEstimator} 
	By setting $T_0 = \sqrt{T}$, after $T_0$ rounds of Algorithm~\ref{alg:meta alg} we have that $\hat \rho\le \rho$ with probability at least $1-\delta$.
\end{restatable}

To recover a good estimate of $\rho$, we need the value of $\rho$ to be sufficiently large. 
Formally, we consider the following condition.\footnote{Notice that even if $\rho$ does not satisfy the condition, $\hat\rho$ is a lower bound on $\rho$. This is sufficient to guarantee that the results in Theorem~\ref{thm:withoutSlater} and Theorem~\ref{thm:sto} hold.}
\begin{condition}\label{ass:strictly}
	It holds that $\rho\ge \frac{2}{T_0} \left(2\err_{ T_0,\delta} + 2\cump_{T_0,\delta}+ \cumd_{T_0}\right)$.
\end{condition}

\begin{remark}
	Notice that, by using primal and dual RMs whose regret bounds are of the order $\tilde O(\sqrt{T})$, and setting $T_0=\sqrt{T}$ Condition~\ref{ass:strictly} is satisfied when $\rho = \omega(T^{-1/4})$.
\end{remark}

Next, we show that $\hat \rho = \Omega(\rho)$, which allows us to exploit the guarantees proved for Algorithm~\ref{alg:meta alg known} in order to provide analogous ones for Algorithm~\ref{alg:meta alg}.
Formally:
\begin{restatable}{lemma}{lemmaSuperGoodEstimator}
	By setting $T_0 = \sqrt{T}$, and assuming that Condition~\ref{ass:strictly} is satisfied, after $T_0$ rounds of Algorithm~\ref{alg:meta alg} we have that $\hat \rho  \ge \rho/2$ with probability at least $1-2\delta$.
\end{restatable}

By applying the results of the previous sections on the guarantees of Algorithm~\ref{alg:meta alg known}, and by using primal and dual RMs whose regret bounds are of the order $\tilde O(\sqrt{T})$, we get $\tilde O(\sqrt{T}/  \rho)$ and $\tilde O(\sqrt{T}/ \rho)$ regret and violation bounds, respectively, when the functions $g_t$ are selected stochastically.

\section{Applications to repeated auctions settings} \label{sec:app}

Internet advertising platforms usually operationalize large auction markets by using \emph{proxy bidders} that place bids in repeated auctions on the advertisers' behalf.
A proxy-bidder selects bids according to a \emph{budget-pacing mechanism}, which manages the usage of the advertisers' budget over time \cite{agarwal2014budget,conitzer2021multiplicative,balseiro2021budget}.
In this section, we discuss the application of our framework to budget-management in auctions, arguing that it can deal with more general constraints on ad slots allocation with respect to what is currently achievable with multiplicative pacing algorithms, which manage only \emph{knapsack constraints}.

We consider the problem faced by a bidder who takes part in a sequence of repeated auctions.
We focus on the case of \emph{second-price}  and \emph{first-price} auctions, since they are the \emph{de facto} standard in large Internet advertising platforms. 
At each round $t\in [T]$, the bidder observes their valuation $v_t$ from a finite set of $\nval$ possible valuations $\cV\subset [0,1]$.
Such valuation models targeting preferences of the advertiser.
Then, the bidder chooses a bid $b_t\in \cB$, where $ \cB \subset [0,1]$ is a finite set of $\nbid$ possible bids such that $0\in\cB$ (\emph{i.e.}, the bidder is allowed to skip items without incurring in any cost).
The utility of the bidder depends on the largest among competing bids, denoted by $\beta_t$.
In particular, the utility is computed as $f_t(b_t) = (v_t-c_t(b_t))\One[b_t\ge \beta_t]$, where the cost $c_t$ is such that $c_t(b_t)= \beta \One[b_t\ge \beta_t]$ in second-price auctions, and $c_t(b_t)= b_t\One[b_t\ge \beta_t]$ for first-price ones.
Finally, the bidder has a target \emph{per-round} budget of $\rho>0$, which yields an overall budget $B\defeq\rho T $ that limits the total spending over the $T$ rounds. 
In the case of budget-constrained bidding, a strictly feasible solution can be easily achieved by always bidding $0$.
Using the target per-round budget $\rho=B/T$ we can write the budget constraint  as $\sum_{t\in[T]}g_t(b_t)\le 0$, with $g_t(b) = c_t(b) - \rho$ for any $b\in\cB$.
Notice that, in this setting, we have the same feasibility parameter $\rho$ for both the stochastic and the adversarial case.

As a benchmark to evaluate the algorithm, we consider the best feasible static policy $\pi:\cV\rightarrow \cB$. 
The set of static policies can be represented by $\cX\defeq \cB^{\nval}$, where a vector $\vb \in \cB^{\nval}$ encodes the policy's bids for each possible valuation.
To apply our framework to this problem, it is sufficient to design a primal regret minimizer constructor (recall that, in order to design dual RMs, we can employ OMD).
This can be implemented by instantiating a regret minimizer \textsc{Exp3.P}~\citep{auer2002nonstochastic} for each possible valuation in $\cV$.
Given a failure probability $\nu\in (0,1)$, each RM guarantees a regret bound $O(\sqrt{T \nbid \log(\nbid/\nu)})$ with probability at least $1-\nu$.
Thus, given a desired failure probability $\eta \in (0,1)$, by setting $\nu=\eta/\nval$ we get that, with probability at least $1-\eta$, the bounds of all the RMs hold.
Hence, 
by a union bound,
we get that the regret of a primal RM is $\cump_{T,\eta}= O(\nval \sqrt{T \nbid \log(\nbid \nval/\eta)} )$. 

\paragraph{Guaranteed budget completion in the stochastic case.}
The crux of budget-pacing mechanisms is ensuring  that the advertisers' budget is not depleted too early (thereby missing potentially valuable future advertising opportunities), while being fully depleted within the planned duration of the campaign. 
	Theorem~\ref{thm:sto} shows that, when inputs are generated according to some stochastic model, Algorithm~\ref{alg:meta alg known} never enters the recovery phase.
	This is crucial in the context of budget-pacing mechanisms, because whenever the algorithm enters the recovery phase it will converge to always bid $0$ in order to mitigate constraints violation.
	Therefore, the bidder could miss out on potentially valuable items.
	Moreover, if the platform wanted to guarantee that the bidder does not spend more than the budget $B$, it would be enough to set a \emph{virtual budget} of $B-\tilde O(T^{1/2})$ to compensate for the potential constraints violation.
	Finally, we argue that, in large-scale markets, an individual bidder has almost no impact on prices, and, thus, stochastic behavior of costs is a reasonable assumption.

\paragraph{Adversarial case.}
	Theorem 6.1 of~\cite{castiglioni2022Online} shows how to construct an algorithm that provides a $\rho$ fraction of the optimal utility for problems with budget constraints and adversarial inputs. The ratio $\rho/(1+\rho)$ obtained in Theorem~\ref{thm:advAdv} matches such result. The latter assumes that rewards and costs are in $[0,1]$, and, thus, $g_t \in[-\rho,1-\rho]$ (as they only model budget constraints). However, in our case we have $g_t \in[-1,1]$. By normalizing the former range to match with ours, we get $g_t \in[-\rho/(1-\rho),1]$. Therefore, the feasibility parameter would be $\rho'=\rho/(1-\rho)$. By rewriting our guarantees as a function of $\rho$, we get $\rho'/(1+\rho')=\rho$, which is the same guarantee of~\citep{castiglioni2022Online}.

\paragraph{Handling ROI constraints.}
Traditional budget-pacing mechanisms (see, \emph{e.g.},~\citep{balseiro2019learning,balseiro2020Arxiv}) are based on primal-dual algorithms that are near optimal in settings with knapsack constraints only, and they cannot be generalized to deal with other types of long-term constraints.
However, there are many real-world situations in which guaranteeing other types of constraints is crucial for practical applications (see, \emph{e.g.},~\citep{golrezaei2021auction,golrezaei2021bidding}).

One example is the case of \emph{return on investment} (ROI) constraints~\cite{auerbach2008empirical,golrezaei2021auction,li2020incentive}.~\footnote{This is a frequent advertising objective in large Internet advertising platform. See, \emph{e.g.}, \url{https://tinyurl.com/c86rezhd} and \url{https://tinyurl.com/mr49vz8a}.} The recent work by \citet{golrezaei2021bidding} presents a threshold-based algorithms for repeated \emph{second-price} auctions under budget and ROI constraints. Our framework allows advertisers to reach a target ROI while keeping budget expenditures under control also in the setting  of repeated \emph{first-price} auctions, which is a frequent setting in practice.\footnote{For example, in 2019 Google announced a shift to first-price auctions for its AdManager exchange. See \url{ https://tinyurl.com/chv5nxys}.}
In particular, given a target ROI $\omega \ge 0$ and the largest among competing bids $\beta_t$, we define the ROI constraints as 
\[
g_t(b_t)=\mleft(\omega-\frac{v_t}{b_t}\mright)\One[b_t\ge \beta_t]\le 0.
\]
Then, it is enough to instantiate the framework with the same setup of Section~\ref{sec:app}, that is,  \textsc{Exp3.P}~\citep{auer2002nonstochastic} for each of the possible valuations in $\cV$, and OMD equipped with negative-entropy regularizer for the dual RM. Therefore, we immediately get that the cumulative violations of the budget and ROI constraints are upper bounded by $\tilde O(T^{1/2})$. This holds both in the fully stochastic and in the fully adversarial setting under the assumption of having a strictly feasible solution, which is reasonable since it is enough to have a sufficiently \emph{small} bid in the set of available bids $\cB$.
We observe that  always bidding such a \emph{small} bid is sufficient to satisfy the ROI constraints but will penalize the cumulative rewards obtained by the advertiser.

\paragraph{Future research direction: fairness constraints.}

Consider the setting in which each item appearing at time $t$ is characterized by one or more of $\ncat$ categories according to the vector $\ve_t\in[0,1]^{\ncat}$. A bidder may have distributional preferences over such categories, such as ensuring that at least a certain fraction of impressions is allocated to each category.  This is the case, for example, of advertisers who need to perform online outreach to a population of users while achieving a distribution over different demographics \emph{close} to that of the real underlying population. For example, \citet{gelauff2020advertising} provide an interesting field study about running advertising campaigns for Participatory Budgeting elections. In Participatory Budgeting elections, community members are asked to vote between various public projects in order to allocate a total budget. The election organizer may use online advertising to try to promote the initiative, and in doing so the goal is to reach a ``demographic mix'' comparable to that of the local population. Surprisingly, \citet{gelauff2020advertising} show that advertisers currently have to resort to complex segmentation strategies through subcampaigns in order to achieve that goal. 

Two recent works propose to achieve such distributional preferences within budget-pacing mechanisms by embedding them into a concave regularization term in the advertiser's objective \cite{balseiro2021regularized,celli2022parity}. Such frameworks specifically consider the case of repeated second-price auctions, and can directly handle only packing constraints. Encoding distributional preferences via a regularization term in the objective implies that they cannot provide any formal guarantee w.r.t. how \emph{close} the realized distribution of impressions is to the target, despite showing promising performance in practice. 

Differently from previous  work, our framework can \emph{explicitly} handle distributional constraints within second- and first-price auction frameworks. Let vector $\hat \ve \in[0,1]^{\ncat}$ be such that $\hat \ve_j$ is the fraction of impressions that we want to be allocated to users of category $j$. 
Then, for each category $j\in [\ncat]$, we could enforce the following type of constraints
\[
g_{t,j}(b_t)\defeq \hat \ve_j -  \ve_{t,j} \One[b_t\ge\beta_t]\le 0.
\]

Assuming the existence of a strictly feasible bidding strategy, our framework guarantees that, for each category $j$, 
\[
\hat \ve_j -  \frac{1}{T}\sum_{t=1}^T \ve_{t,j} \One[b_t\ge\beta_t]\le \tilde O(T^{-1/2}),
\]
which guarantees that, in the limit, the difference between the average distribution of impressions and the target thresholds is vanishing. 

The main question which still needs to be answered in order to apply our framework in the case of fairness constraints is whether we can motivate the existence of a strictly feasible solution. One reasonable requirement is to constrain the target vector $\hat \ve$ to be a point in the full-dimensional simplex with dimension $\ncat$. On top of that, the advertiser would need a way to ``buy what's necessary'' in order to match the distributional constraints. 
This desideratum could be achieved, for example, by introducing \emph{buyout options} for advertisers, in the spirit of \citet{gallien2007temporary} (\emph{i.e.}, when the advertiser needs impression from a certain category, they always have the option of bidding the buyout value to be sure to win the relevant items). Therefore, assuming the population of users is large enough, an advertiser could achieve a strictly feasible solution by bidding according to the fixed strategy mixture recommending to bid the buyout option for each category $j$ with a probability greater than or equal to $\hat\ve_j$. 

The model we described is clearly a simplification of real budget-pacing systems. Moreover, the practical implications of introducing buyout options should be further investigated, in order to understand if they constitute a viable solution both for the platform and advertisers.  Finally, we leave as interesting future research directions the problem of studying the general setting (with arbitrary sets $\cV$ and $\cB$), and that of providing an empirical evaluation of the above techniques on real-world data.

\bibliographystyle{plainnat}
\bibliography{bib-abbrv,refs,lit}

\clearpage

\appendix

\section{Omitted proofs}\label{sec:omitted proofs}

\subsection{Proof omitted from Section~\ref{sec:duality}}

\theoremDuality*

\begin{proof}
	As a first step, we prove that $\inf_{\vlambda\in \cD_{d_g}}  \sup_{\vxi\in\Xi }  \cL(\vxi,\vlambda) = \inf_{\vlambda\in \mathbb{R}_+^m} \sup_{\vxi\in\Xi } \cL(\vxi,\vlambda)$.
	Notice that for each $\vlambda \in \mathbb{R}_+^m$ such that $\|\vlambda \|_1 > 1/ d_g$, it holds
	\[
	\sup_{\vxi  \in \Xi}  \cL(\vxi,\vlambda)\ge  \cL(\vxi^\circ,\vlambda)\ge - \langle\vlambda^*, g(\vxi^\circ) \rangle  \geq d_g \| \vlambda^* \|_1 > 1,
	\]
	where, with an abuse of notation, $\vxi^\circ \in \Xi$ denotes a {strictly feasible} strategy mixture for Problem~\ref{eq:opt lp gen}. That is a strategy mixture $\vxi \in \Xi$ which is optimal for the problem defining $d_g$ in Equation~\eqref{eq:dg}, and, thus, it satisfies all the constraints by at least $d_g$ (\emph{i.e.}, it holds $g_i(\vxi^\circ) \leq -d_g$ for all $i \in [m]$).\footnote{Notice that $\vxi^\circ$ may \emph{not} be well defined when the problem in Equation~\ref{eq:dg} does \emph{not} admit a maximum. In such cases, we can take a $\vxi^\circ$ that is  arbitrarily ``close'' to a supremum, so that the result still holds.}
	Thus, it holds that
	\[
	\inf_{\vlambda\in \mathbb{R}_+^m \setminus \cD_{d_g}} \sup_{\vxi  \in \Xi} \, \cL(\vxi,\vlambda) >1.
	\] 
	Moreover, since 
	\[
	 \inf_{\vlambda\in \cD_{d_g}} \sup_{\vxi  \in \Xi}  \, \cL(\vxi,\vlambda)\le\sup_{\vxi  \in \Xi}  \cL (\vxi,\mathbf{0})\le  1,
	\]
	we can conclude that
	\begin{align} \label{eq:dualityOne}
		\inf_{\vlambda\in \mathbb{R}_+^m} \sup_{\vxi  \in \Xi}  \, \cL(\vxi,\vlambda) &= \min \left\{ \inf_{\vlambda\in \cD_{d_g}} \sup_{\vxi  \in \Xi}  \, \cL(\vxi,\vlambda); \inf_{\vlambda\in \mathbb{R}_+^m \setminus \cD_{d_g}} \sup_{\vxi  \in \Xi} \, \cL(\vxi,\vlambda)  \right\} \nonumber \\
		&= \inf_{\vlambda\in \cD_{d_g}} \sup_{\vxi  \in \Xi}  \, \cL(\vxi,\vlambda).
	\end{align}	
	Then, 
	\begin{align*}
	\OPT_{f,g} & =\sup_{\vxi\in\Xi } \inf_{\vlambda\in \mathbb{R}_+^m} \cL(\vxi,\vlambda) \\
	& \le \sup_{\vxi\in\Xi } \inf_{\vlambda\in \cD_{d_g}} \cL(\vxi,\vlambda) \\
	& \le \inf_{\vlambda\in \cD_{d_g}} \sup_{\vxi\in\Xi }  \cL(\vxi,\vlambda) \\
	& = \inf_{\vlambda\in \mathbb{R}_+^m} \sup_{\vxi\in\Xi}  \cL(\vxi,\vlambda)  \\
	&  =\OPT_{f,g},
	\end{align*}
	where the first inequality holds since in the right-hand side the $\inf$ is taken over the more restrictive set $\cD_{d_g}$, the second one by the \emph{max–min inequality}, while the second-to-last equality holds by Equation~\eqref{eq:dualityOne}.
	This concludes the proof.
\end{proof}

\subsection{Proofs omitted from Section~\ref{sec:advStoc}}

\lemmaAzuma*

\begin{proof}
	Given a desired failure probability $\delta \in (0,1)$, recall that  $\eta=\delta/3$ and set $\eps = \eta / 18 m T^2$. 
	Consider the following inequalities in which the differences between expectations and empirical means of constraint functions are bounded:
	\begin{align}
		&\mleft|\sum_{\tau = t}^{t'}  g_{\tau,i} (\vx_{\tau})-\sum_{\tau = t}^{t'} \bar g_i (\vx_{\tau})\mright|> 2 \sqrt{2(t'-t)\ln \frac{1}{\eps}}  ,\label{in:1}\\
		&\mleft|\sum_{\tau = t}^{t'}  g_{\tau,i} (\vxi^\circ)-\sum_{\tau = t}^{t'} \bar g_i (\vxi^\circ) \mright|> 2 \sqrt{2(t'-t)\ln \frac{1}{\eps}}, \label{in:2}\\
		&\mleft|\sum_{\tau = t}^{t'}  g_{\tau,i} (\vxi^*)-\sum_{\tau = t}^{t'} \bar g_i (\vxi^*)\mright|> 2 \sqrt{2(t'-t)\ln \frac{1}{\eps}} , \label{in:3} \\
		&\mleft|\sum_{\tau = t}^{t'} \vlambda_\tau g_{\tau,i} (\vx_\tau)-\sum_{\tau = t}^{t'} \vlambda_\tau  \bar{g}_i (\vx_\tau)\mright|> 2 \max_{\tau \in [T]: t\leq \tau \leq t'} ||\vlambda_\tau||_1 \sqrt{2(t'-t)\ln \frac{1}{\eps}}, \label{in:4}\\
		&\left| \sum_{\tau = t}^{t'}  \vlambda_\tau \, g_{\tau,i} (\vxi^*)-\sum_{\tau = t}^{t'} \vlambda_\tau \, \bar{g}_i (\vxi^*) \right|> 2 \max_{\tau \in [T]: t\leq \tau \leq t'} ||\vlambda_\tau||_1 \sqrt{2(t'-t)\ln \frac{1}{\eps}},\\
		&\left| \sum_{\tau = t}^{t'}  \vlambda_\tau \, g_{\tau,i} (\vxi^\circ)-\sum_{\tau = t}^{t'} \vlambda_\tau \, \bar{g}_i (\vxi^\circ) \right|\le 2  \max_{\tau \in [T]: t\leq \tau \leq t'} ||\vlambda_\tau||_1  \sqrt{2(t'-t)\ln \frac{1}{\eps}}.
	\end{align}
By applying Azuma-Hoeffding inequality to each martingale difference sequence, we get that each inequality holds with probability at most $2 \eps$.
We denote by $\CE_\eta$ the event in which Equations~\eqref{in:1},~\eqref{in:2},~\eqref{in:3},~and~\eqref{in:4} are satisfied for all $ t,t' \in [T]$ with $t<t'$, and for all $i \in [m]$.
By a union bound that takes into account the six events above, the $m$ constraints, and all the possible time intervals from $t$ to $t'$, which are at most $T^2$, we get:
\begin{align*}
	\mathbb{P}\mleft( \mathbb{\CE_\eta} \mright) \geq 1- 6mT^2(2\eps) = 1- 12mT^2\eps = 1 - \frac{2}{3} \eta \geq 1- \eta .
\end{align*}
Therefore, event $\CE_\eta$ holds with probability at least $1 - \eta$.
Moreover, let us recall that:
\begin{align*}
\err_{t'-t,\eta} = \sqrt{8(t'-t)\ln \frac{18mT^2}{\eta}} = 2 \sqrt{2(t'-t)\ln \frac{1}{\eps}}.
\end{align*}

Now, consider event $\CE$ in which Algorithms~\ref{alg:meta alg known} satisfies the following conditions: (i) the regret incurred by $\cRpuno$ after $T_1$ rounds is upper bounded by $\cump_{T_1,\eta}$; (ii) the regret cumulated by $\cRpdue$ after the remaining $T - T_1$ rounds is upper bounded by $\cump_{T- T_1,\eta}$; and (iii) event $\CE_\eta$ holds.
Recall that each one of the conditions (i), (ii) and (iii) holds with probability at least $1 - \eta$; hence, by a union bound we get:
\begin{align*}
	\mathbb{P}\mleft( \CE \mright) \geq 1 - 3 \eta = 1-\delta.
\end{align*}
This concludes the proof.
\end{proof}

\lemmaSmallRegret*

\begin{proof}
	By the no-regret property of the primal regret minimizer, we have that:
	\begin{align} \label{eq:no reg prim}
		\sum_{t = 1}^{T_1} \Big(  f_t(\vx_t) - \langle \vlambda_t, g_t(\vx_t) \rangle \Big) \ge \sum_{t =1}^{T_1}\Big( f_t(\vxi^*) - \langle \vlambda_t , g_t(\vxi^*) \rangle \Big)  -\mleft(1+\frac{2}{\tilde \rho}\mright)\cump_{T_1,\eta}.
	\end{align}
	Let $i^\star \in\argmax_{i \in [m]} \sum_{t = 1}^{T_1} g_{t,i}(\vx_t)$ be one of the ``most violated'' constraints.
	We prove that: 
	\begin{align} \label{eq:gain0}
		\sum_{t = 1}^{T_1} \langle \vlambda_t ,g_t(\vx_t) \rangle \ge (T-T_1) - \frac{1}{\tilde \rho} \cumd_{T_1}.
	\end{align}
	To do that, we consider the following two cases.
	
	\paragraph{Case $T_1=T$.}
	We get:
	\begin{align*}
	\sum_{t=1}^{T_1} \langle \vlambda_t, g_t(\vx_t) \rangle \ge \sum_{t=1}^{T_1}  \langle \mathbf{0},  g_{t}( \vx_t)\rangle  - \frac{1}{\tilde \rho} \cumd_{T_1} =  (T-T_1) - \frac{1}{\tilde \rho} \cumd_{T_1}.
	\end{align*}

    \paragraph{Case $T_1 < T$.}
    By the condition in Line~\ref{line:while} of Algorithm~\ref{alg:meta alg known}, we have that $ \sum_{t = 1}^{T_1} g_{t,i^\star}(\vx_t) \ge (T-T_1)\tilde \rho$. Thus, we have that: 
	\begin{align*}
		\sum_{t=1}^{T_1} \langle \vlambda_t, g_{t}(\vx_t) \rangle \ge \sum_{t=1}^{T_1} \frac{1}{\tilde \rho} g_{t,i^\star}(\vx_t)- \frac{1}{\tilde \rho} \cumd_{T_1}\ge (T-T_1) - \frac{1}{\tilde \rho} \cumd_{T_1},
	\end{align*}
    where the first inequality follows from the no-regret property of the dual regret minimizer and the second one from the fact that $ \sum_{t = 1}^{T_1} g_{t,i^\star}(\vx_t) \ge (T-T_1)\tilde \rho$ when $T_1 < T$.
    
	Now, by using Equation~\eqref{eq:gain0}, we can provide a lower bound on the cumulative reward obtained by Algorithm~\ref{alg:meta alg known} during the play phase. 
	We have that:
	\begin{align*}
		\sum_{t=1}^{T_1} f_t(\vx_t) &\ge \sum_{t=1}^{T_1} \Big(  f_t(\vxi^*) - \langle \vlambda_t, g_t(\vxi^*) \rangle + \langle \vlambda_t, g_t(\vx_t)\rangle \Big) - \mleft(1+\frac{2}{\tilde \rho}\mright)\cump_{T_1,\eta},\\
		& \ge \sum_{t=1}^{T_1} \Big( f_t(\vxi^*) - \langle \vlambda_t,  g_t(\vxi^*) \rangle \Big)+(T-T_1)- \mleft(1+\frac{2}{ \tilde \rho}\mright) \cump_{T_1,\eta} - \frac{1}{\tilde \rho} \cumd_{T_1},\\
		& \ge \sum_{t=1}^{T_1} \Big( f_t(\vxi^*) - \langle \vlambda_t, \bar g(\vxi^*) \rangle \Big)+(T-T_1)- \mleft(1+\frac{2}{ \tilde \rho}\mright) \cump_{T_1,\eta}- \frac{1}{\tilde \rho} \cumd_{T_1}-\frac{1}{\tilde{\rho}} \err_{T_1,\eta},\\
		& \ge \sum_{t=1}^{T_1}  f_t(\vxi^*) +(T-T_1)- \mleft(1+\frac{2}{\tilde \rho}\mright) \cump_{T_1, \eta}- \frac{1}{\tilde \rho} \cumd_{T_1}-\frac{1}{\tilde{\rho}} \err_{T_1,\eta},
	\end{align*}
	where the first inequality holds by Equation~\eqref{eq:no reg prim}, the second one by Equation~\eqref{eq:gain0}, the third one follows from the fact that the event $\CE$ holds, while the last one from the fact that $\bar g(\vxi^*)\le 0$ by definition.
\end{proof}

\lemmaSmallViolation*

\begin{proof}
	Let $i^\star \in \argmax_{i \in [m]} \sum_{t = T_1+1}^{T} g_{t,i}(\vx_t)$ be one of the ``most violated'' constraints. Then, 
	\begin{align*}
		(T-T_1)\rho &\le -\sum_{t = T_1+1}^T \langle \vlambda_t, \bar g(\vxi^\circ)\rangle  \\
		&\le -\sum_{t = T_1+1}^T \langle \vlambda_t  g_t(\vxi^\circ) \rangle + \err_{T-T_1,\eta}\\
		&\le - \sum_{t = T_1+1}^T \langle \vlambda_t g_t(\vx_t) \rangle +2 \cump_{T-T_1, \eta}+ \err_{T-T_1,\eta} \\
		&\le  - \sum_{t = T_1+1}^T  g_{t,i^\star}(\vx_t) +\cumd_{ T-T_1} +2 \cump_{T-T_1,\eta}+ \err_{T-T_1,\eta},
	\end{align*}
	where the first inequality comes from the definition of $\rho$, the second one from the fact that event $\CE$ holds, the third one from the no-regret property of the primal regret minimizer, and the last one from the no-regret property of the dual regret minimizer.
	Hence,
		\begin{align} \label{eq:boooo}
			 \sum_{t = T_1+1}^T g_{t,i^\star}(\vx_t)	\le -(T-T_1)\rho - \cumd_{ T-T_1} +2 \cump_{T-T_1,\eta}+ \err_{T-T_1,\eta}.
		\end{align}
	It follows from the definition of $i^\star$ that, if Equation~\eqref{eq:boooo} holds for $i^\star$, then, it holds for every $i \in [m]$. 
	This concludes the proof.
\end{proof}

\theoremAdvStocSlater*

\begin{proof}
	By Lemma \ref{lm:ce}, event $\CE$ holds with probability at least $1 - \delta$. 
	In the rest of the proof, we assume the event $\CE$ holds, providing a bound that holds with probability at least $1-\delta$.
	
	We first provide an upper bound on the cumulative regret.
	By Lemma \ref{lm:smallReg}, we have:
	\begin{align} \label{eq:reg slater}
		\sum_{t = 1}^{T_1} f_t(\vx_t) \ge \sum_{t = 1}^{T_1} f_t(\vxi^*) + (T-T_1) -\frac{1}{\tilde \rho}\err_{T_1,\eta}-  \mleft(1+\frac{2}{\tilde \rho}\mright)\cump_{T_1,\eta}-\frac{1}{\tilde{\rho}}\cumd_{T_1}. 
	\end{align}
	Hence, it holds:
	\begin{align*} 
		\sum_{t = 1}^T f_t(\vx_t) & \ge \sum_{t = 1}^{T_1} f_t(\vx_t)\\
		&\ge \sum_{t=1}^{T_1} f_t(\vxi^*) + (T-T_1) -\frac{1}{\tilde \rho}\err_{T_1,\eta}-  \mleft(1+\frac{2}{\tilde \rho}\mright)\cump_{T_1,\eta}-\frac{1}{\tilde{\rho}}\cumd_{T_1}\\
		&\ge \sum_{t =1}^T f_t(\vxi^*) -\frac{1}{\tilde \rho}\err_{T_1,\eta}-  \mleft(1+\frac{2}{\tilde \rho}\mright)\cump_{T_1,\eta}-\frac{1}{\tilde{\rho}}\cumd_{T_1}\\
		&\ge \sum_{t =1}^T f_t(\vxi^*) -\frac{1}{\tilde \rho}\err_{T,\eta}-  \mleft(1+\frac{2}{\tilde \rho}\mright)\cump_{T,\eta}-\frac{1}{\tilde{\rho}}\cumd_{T},
	\end{align*}
	where the second inequality holds by Equation~\eqref{eq:reg slater} and the third one by $\sum_{t = T_1+1}^T f_t(\vxi^*) \le T-T_1$, which follows from the fact that the range of $f_t$ is $[0,1]$.

	By recalling that $\vxi^* \in \Xi$ is defined as an optimal solution to Problem~\hyperref[eq:opt lp gen]{$\cP_{\bar f,\bar g}$} and $R^T = T \, \OPT_{\bar f, \bar g} - \sum_{t=1}^T f_t(\vx_t)$, the following bound on the cumulative regret holds:
	\begin{align*} 
		R^T =  \sum_{t =1}^T f_t(\vxi^*) - \sum_{t = 1}^T f_t(\vx_t) \leq \frac{1}{\tilde \rho}\err_{T,\eta} + \mleft(1+\frac{2}{\tilde \rho}\mright)\cump_{T,\eta} + \frac{1}{\tilde{\rho}}\cumd_{T}.
	\end{align*}
	
	Next, we provide an upper bound on the cumulative constraints violation.
	
	By Lemma \ref{lm:smallViolation}, for every $i \in [m]$, we have that: 
	\begin{align} \label{eq:viol slater}
		\sum_{t = T_1+1}^T g_{t,i}(\vx_t)\le -(T-T_1) \rho +2\cump_{T-T_1,\eta} + \cumd_{T-T_1}+  \err_{T-T_1,\eta}.
	\end{align}
	Hence, for every $i \in [m]$, it holds 
	\begin{align*}
		\sum_{t = 1}^T g_{t,i}(\vx_t) &= 
		\sum_{t = 1}^{T_1} g_{t,i}(\vx_t) +\sum_{t = T_1+1}^T g_{t,i}(\vx_t) \\
		&\le  (T-T_1) \tilde \rho + M_{\tilde \rho} -(T-T_1) \rho +2\cump_{T-T_1, \eta} + \cumd_{T-T_1}+  \err_{T-T_1,\eta}\\
		&\le  M_{\tilde \rho}+2\cump_{T-T_1, \eta} + \cumd_{T-T_1}+  \err_{T-T_1,\eta}\\
		&\le  M_{\tilde \rho} +2\cump_{T,\eta} + \cumd_{T}+  \err_{T,\eta}.
	\end{align*}
The first inequality follows from Equation \eqref{eq:viol slater} and by the condition in Line~\ref{line:while} of Algorithm~\ref{alg:meta alg known}, which ensures $ \sum_{t = 1}^{T_1} g_{t,i}(\vx_t) \leq (T-T_1)\tilde \rho + M_{\tilde \rho}$ for every $i \in [m]$.
Moreover, the second inequality follows from $\tilde \rho\le \rho$, since Condition \ref{ass:strictly stoc} holds.
Let $i^\star \in \argmax_{i \in [m]} \sum_{t = 1}^{T} g_{t,i}(\vx_t)$ be one of the most violated constraints.
By recalling that $V^T = \max_{i \in [m]} \sum_{t =1}^T g_{t,i}(\vx_t)$, the following bound on the cumulative constraints violation holds:
	\begin{align*}
	V^T = \sum_{t = 1}^T g_{t,i^\star}(\vx_t) \le  M_{\tilde \rho} +2\cump_{T,\eta} + \cumd_{T}+  \err_{T,\eta}.
\end{align*}
This concludes the proof.
\end{proof}

\theoremAdvStocWithout*

\begin{proof}
	If $\hat \rho\ge 2T^{-1/4}$, the claim follows by Theorem~\ref{thm:withSlater}.
	Thus, we prove the statement for the case $\tilde \rho=T^{-1/4}$.
	First, we provide an upper bound on the cumulative regret.
	By Lemma \ref{lm:ce}, we have that event the $\CE$ holds with probability at least $1-\delta$. In the rest of the proof, we assume that the event $\CE$ holds, and provide a bound that holds with probability at least $1-\delta$. 
	We have:
	\begin{align*} 
		\sum_{t = 1}^T f_t(\vx_t) & \ge \sum_{t = 1}^{T_1} f_t(\vx_t)\\
		&\ge \sum_{t = 1}^{T_1} f_t(\vxi^*) + (T-T_1) -\frac{1}{\tilde \rho}\err_{T_1,\eta}-  \mleft(1+\frac{2}{\tilde \rho}\mright)\cump_{T_1,\eta}-\frac{1}{\tilde{\rho}}\cumd_{T_1}\\
		&\ge \sum_{t =1}^T f_t(\vxi^*) -\frac{1}{\tilde \rho}\err_{T_1,\eta}-  \mleft(1+\frac{2}{\tilde \rho}\mright)\cump_{T_1,\eta}-\frac{1}{\tilde{\rho}}\cumd_{T_1}\\
		&\ge \sum_{t =1}^{T} f_t(\vxi^*) -\frac{1}{\tilde \rho}\err_{T,\eta}-  \mleft(1+\frac{2}{\tilde \rho}\mright)\cump_{T,\eta}-\frac{1}{\tilde{\rho}}\cumd_{T}\\
		& \ge \sum_{t =1}^{T} f_t(\vxi^*) -T^{1/4}\err_{T,\eta}-  \mleft(1+2T^{1/4}\mright)\cump_{T,\eta}-T^{1/4}\cumd_{T}.
	\end{align*}
    These steps are similar to those used to prove the regret bound in Theorem~\ref{thm:withSlater} (see the proof of Theorem~\ref{thm:withSlater} for further details).
    By recalling that $\vxi^*$ is an optimal solution to Problem~\hyperref[eq:opt lp gen]{$\cP_{\bar f,\bar g}$} and $R^T = T \, \OPT_{\bar f, \bar g} - \sum_{t=1}^T f_t(\vx_t)$, the following bound on the cumulative regret holds:
    \begin{align*} 
    	R^T =  \sum_{t =1}^T f_t(\vxi^*) - \sum_{t = 1}^T f_t(\vx_t) \leq T^{1/4}\err_{T,\eta} + \mleft(1+2T^{1/4}\mright)\cump_{T,\eta}+ T^{1/4}\cumd_{T}.
    \end{align*}
    
    Next, we provide an upper bound on the cumulative constraints violation.

	For every $i \in [m]$, the following holds 
		\begin{align}
			\sum_{t = T_1+1}^T g_{t,i}(\vx_t)&\le -(T-T_1)  \rho  +2\cump_{T-T_1, \eta} + \cumd_{T-T_1}+ \err_{T-T_1,\eta} \nonumber \\
			&\le  2\cump_{T-T_1,\eta} + \cumd_{T-T_1}+ \err_{T-T_1,\eta}, \label{eq:no slater}
		\end{align}
	where the first inequality follows from Lemma~\ref{lm:smallViolation}, while the second one from $\rho\ge 0$.
	Hence, for every $i \in [m]$, it holds
	\begin{align*}
		\sum_{t=1}^T g_{t,i}(\vx_t) &= 
		\sum_{t =1}^{T_1} g_{t,i}(\vx_t) +\sum_{t = T_1+1}^T g_{t,i}(\vx_t) \\
		&\le (T-T_1)  \tilde{\rho}+ M_{\tilde{\rho}}+2\cump_{T-T_1, \eta} + \cumd_{T-T_1}+ \err_{T-T_1,\eta}\\
		&\le (T-T_1)  T^{-1/4}+ M_{T^{-1/4}}+2\cump_{T-T_1, \eta} + \cumd_{T-T_1}+ \err_{T-T_1,\eta}\\
		&\le  T^{3/4}+M_{T^{-1/4}}+ 2\cump_{T-T_1,\eta} + \cumd_{T-T_1}+ \err_{T-T_1,\eta}\\
		&\le  T^{3/4}+M_{T^{-1/4}}+ 2\cump_{T, \eta} + \cumd_{T}+ \err_{T,\eta}.
	\end{align*}
The first inequality follows from Equation~\eqref{eq:no slater} and from the condition in Line~\ref{line:while} of Algorithm~\ref{alg:meta alg known}, which ensures that $ \sum_{t = 1}^{T_1} g_{t,i}(\vx_t) \leq (T-T_1)\tilde \rho + M_{\tilde \rho}$ for every $i \in [m]$. Moreover, the second inequality follows from $\tilde \rho = T^{-1/4}$.
Thus, by letting $i^\star \in \argmax_{i \in [m]} \sum_{t = 1}^{T} g_{t,i}(\vx_t)$, and by recalling that $V^T = \max_{i \in [m]} \sum_{t =1}^T g_{t,i}(\vx_t)$, the following bound on the cumulative constraints violation holds:
\begin{align*}
	V^T = \sum_{t = 1}^T g_{t,i^\star}(\vx_t) \le  T^{3/4}+M_{T^{-1/4}} +2\cump_{T,\eta} + \cumd_{T}+  \err_{T,\eta}.
\end{align*}
This concludes the proof.
\end{proof}

\subsection{Proof omitted from Section~\ref{sec:stoc}}

First, we provide a preliminary result on the value of the Lagrangian game when primal and dual  players are constrained to specific sets of strategies.
\begin{restatable}{lemma}{strongDualityNew}\label{lm:FarFromOpt}
	Let $ f \in \cF$ and $ g \in \cG$ be such that $d_g>0$.
	Moreover, given any $\epsilon>0$, let $\Xi_{\epsilon,g}\coloneqq \left\{\vxi \in \Xi: \max_{i \in [m]} g_i(\vxi)\ge \epsilon \right\}$.
	The following holds:
	\textnormal{
	\[
	\sup_{\vxi\in \Xi_\epsilon}\, \inf_{\vlambda\in \cD_{d_g/2}} \cL_{f,g}(\vxi,\vlambda) \le \OPT_{ f,  g}- \frac{\epsilon}{d_g}.
	\]}
\end{restatable}

\begin{proof}
	Let $\vxi \in \Xi_{\epsilon,g}$ and $i^\star \in \argmax_{i\in [m]} g_i(\vxi)$.
	Then, 
	\begin{align*} 
	\inf_{\vlambda \in \cD_{d_g/2}} \Big\{ f(\vxi)- \langle\vlambda ,g(\vxi)\rangle \Big\}&= f(\vxi)-\frac{2}{d_g} \, g_{i^\star}(\vxi)\\
	&= \inf_{\vlambda \in \cD_{d_g}} \Big\{ f(\vxi)-\langle\vlambda, g(\vxi)\rangle \Big\} -\frac{1}{d_g}  g_{i^\star}(\vxi)\\
	&\le  \sup_{\vxi\in\Xi}\,\,\inf_{\vlambda\in\cD_{d_g}} \cL_{f,g}(\vxi,\vlambda)  -\frac{1}{d_g}  g_{i^\star}(\vxi)\\
	&\le  \OPT_{f,g}  -\frac{1}{d_g}  g_{i^\star}(\vxi)\\
	&\le \OPT_{f,g} -\frac{\epsilon}{d_g},
	\end{align*}
	where the second inequality follows from Theorem~\ref{thm:duality_restricted}, while the last one holds by the definition of $\Xi_{\epsilon,g}$ and $i^\star$.
\end{proof}

Next, we introduce a new event that extends $\CE$ by considering also the (stochastic) sequence of reward functions $f_t$. Formally, the event is defined as follows.
\begin{definition}
	We denote with $\CES$ the event in which Algorithm~\ref{alg:meta alg known} satisfies the following conditions (recall that $\eta=\delta/3$): (i) event $\CE$ holds; (ii) for every pair of rounds $t,t' \in [T] : t \leq t'$ it holds:
	\begin{itemize}
		\item $|\sum_{\tau = t}^{t'}  f_\tau (\vx_\tau)-\sum_{\tau = t}^{t'} \bar f (\vx_\tau)|\le \err_{t'-t,\eta}$,
		\item $|\sum_{\tau = t}^{t'} f_\tau (\vxi^*)-\sum_{\tau= t}^{t'} \bar f (\vxi^*)|\le \err_{t'-t,\eta}$.
	\end{itemize}
\end{definition}

\begin{restatable}{lemma}{lemmaCleanStoc}\label{lm:cleanEvent}
		After running Algorithm~\ref{alg:meta alg known}, the event $\CES$ holds with probability at least $1-\delta$.
\end{restatable}

\begin{proof}
	Given a desired failure probability $\delta \in (0,1)$, recall that $\eta= \delta/3$ and set $\eps = \eta/12mT^2$.
	Consider the following inequalities in which the differences between expectations and empirical means of reward functions are bounded:
	\begin{align}
	&\mleft|\sum_{\tau = t}^{t'}  f_{\tau} (\vx_{\tau})-\sum_{\tau = t}^{t'} \bar f (\vx_{\tau})\mright|> 2 \sqrt{2(t'-t)\ln \frac{1}{\eps}}, \label{in:5}\\
	&\mleft|\sum_{\tau = t}^{t'}  f_{\tau} (\vxi^*)-\sum_{\tau = t}^{t'} \bar f (\vxi^*)\mright|> 2 \sqrt{2(t'-t)\ln \frac{1}{\eps}} .\label{in:6}
	\end{align}
	By applying the Azuma-Hoeffding inequality to each martingale difference sequence, we get that each inequality holds with probability at most $2 \eps$. We denote by $\CES_\eta$ the event in which Equations~\eqref{in:5}~and~\eqref{in:6} hold for every $t\le t' \in [T] : t<t'$ and event $\CE_\eta$ holds (see the proof of Lemma~\ref{lm:ce} for the definition of event $\CE_\eta$).
	By a union bound, we have that:
	\begin{align*}
	\mathbb{P}\mleft( \mathbb{\CES_\eta} \mright) \geq 1- 2\eps(4mT^2+2T^2) \geq 1- \eta .
	\end{align*}
	Therefore, event $\CES_\eta$ holds with probability at least $1 - \eta$.
	Moreover, let us recall that:
	\begin{align*}
	\err_{t'-t,\eta} = \sqrt{8(t'-t)\ln \frac{12mT^2}{\eta}} = 2 \sqrt{2(t'-t)\ln \frac{1}{\eps}}.
	\end{align*}
	
	Now, consider the event $\CES$ in which Algorithm~\ref{alg:meta alg known} satisfies the following conditions: (i) the regret incurred by $\cRpuno$ after $T_1$ rounds is upper bounded by $\cump_{T_1,\eta}$; (ii) the regret cumulated by $\cRpdue$ after the remaining $T - T_1$ rounds is upper bounded by $\cump_{T- T_1,\eta}$; and (iii) event $\CES_\eta$ holds.
	Recall that each one of the conditions (i), (ii) and (iii) holds with probability at least $1 - \eta$; hence, by a union bound we get:
	\begin{align*}
	\mathbb{P}\mleft( \CES \mright) \geq 1 - 3 \eta = 1-\delta.
	\end{align*}
	This concludes the proof.
\end{proof}

As a first step, we prove that the primal regret minimizer gets a per-round utility that is close to the value $\OPT_{\bar f, \bar g}$.
Formally:

\begin{restatable}{lemma}{lemmaStocPrimal} \label{lm:stocstocuno}
	If the event $\CES$ holds, then, for every round $\tau\in [T_1]$ the following inequality holds:
	\textnormal{
	\begin{align*}
		\sum_{t = 1}^\tau  \cL_{f_t, g_t}(\vx_t,\vlambda_t) \ge \tau \, \OPTLP_{\bar f, \bar g}- \mleft(1+\frac{2}{\tilde \rho}\mright)\cump_{\tau,\eta}-\mleft(1+\frac{1}{\tilde \rho}\mright) \err_{\tau,\eta}.
	\end{align*}}
\end{restatable}

\begin{proof}
	Let $\vxi^\star$ be an optimal solution to Problem~\hyperref[eq:opt lp gen]{$\cP_{\bar f,\bar g}$}, and let $\bar \vlambda= \frac{1}{\tau}\sum_{t = 1}^\tau \vlambda_t$. 
	Then, it holds
	\begin{align*}
	\sum_{t=1}^\tau \cL_{f_t,g_t}(\vx_t,\vlambda_t) & \ge \sum_{t=1}^\tau \cL_{f_t,g_t}(\vxi^\ast,\vlambda_t) - \mleft(1+\frac{2}{\tilde \rho}\mright)\cump_{\tau,\eta} \\
	& \ge \sum_{t=1}^\tau \cL_{\bar f, \bar g}(\vxi^\ast,\vlambda_t) - \mleft(1+\frac{2}{\tilde \rho}\mright)\cump_{\tau,\eta}-\mleft(1+\frac{1}{\tilde \rho}\mright) \err_{\tau,\eta} \\
	& = \sum_{t=1}^\tau \cL_{\bar f, \bar g}(\vxi^\ast,\bar \vlambda) - \mleft(1+\frac{2}{\tilde \rho}\mright)\cump_{\tau,\eta}-\mleft(1+\frac{1}{\tilde \rho}\mright) \err_{\tau,\eta} \\
	&\geq  \tau \inf_{\vlambda \in \cD_{\tilde \rho}}   \cL_{\bar f,\bar g}(\vxi^*,\vlambda)- \mleft(1+\frac{2}{\tilde \rho}\mright)\cump_{\tau,\eta}-\mleft(1+\frac{1}{\tilde \rho}\mright) \err_{\tau,\eta}\\
	& = \tau \sup_{\vxi\in\Xi } \inf_{\vlambda\in \cD_{\tilde \rho}}  \cL_{\bar f,\bar g} (\vxi,\vlambda) - \mleft(1+\frac{2}{\tilde \rho}\mright)\cump_{\tau,\eta}-\mleft(1+\frac{1}{\tilde \rho}\mright) \err_{\tau,\eta}\\
	& = \tau \OPTLP_{\bar f, \bar g}- \mleft(1+\frac{2}{\tilde \rho}\mright)\cump_{\tau,\eta}-\mleft(1+\frac{1}{\tilde \rho}\mright) \err_{\tau,\eta},
	\end{align*}
	where the first inequality follows from the no-regret property of the primal regret minimizer, the second one from the definition of the event $\CES$, and the third one from the definition of $\vxi^\ast$. Moreover, the first equality follows from the fact that $\bar g$ is independent from $t$.
	This concludes the proof.
\end{proof}

Now, we show that the dual regret minimizer gets a per-round utility that is close to the value $\OPT_{\bar f, \bar g}$. 
Moreover, the attained utility increases by an additive factor proportional to the primal violation.
This can be proved only in the setting with stochastic reward functions. 
Indeed, in this setting the primal and dual regret minimizers are playing a stochastic repeated zero-sum game that converges to an equilibrium. 
Notice that this is \emph{not} true when the reward functions are adversarial.
\begin{restatable}{lemma}{lemmaStocDual} \label{lm:stocstocdue}
	If event $\CES$ holds and Condition \ref{ass:strictly stoc} is satisfied, then for each $\tau\in [T_1]$ and each $i \in [m]$
	\textnormal{
	\begin{align*}
		\sum_{t=1}^\tau \cL_{f_t,g_t}(\vxi_t,\vlambda_t)  \le \tau \OPTLP_{\bar f, \bar g} +\frac{1}{\tilde \rho} \cumd_\tau+\mleft(1+\frac{2}{\tilde\rho}\mright) \err_{\tau,\eta}  - \sum_{t=1}^\tau g_{t,i}(\vx_t).
	\end{align*}}
\end{restatable}

\begin{proof}
	In the following, let $\vlambda^*\in \argmin_{\vlambda \in \cD_{\tilde \rho}} \sum_{t=1}^\tau  \cL_{\bar f, \bar g}(\vxi_t,\vlambda) $, $\epsilon \coloneqq \frac{\max_{i \in [m]}\sum_{t =1}^\tau g_{t,i}(\vx_t)- \err_{\tau,\eta}}{\tau}$, and $\bar \vxi \coloneqq \frac{1}{\tau}\sum_{t = 1}^{\tau} \vxi_t$, where $\vxi_t \in \Xi$ denotes the strategy mixture that plays deterministically $\vx_t$. 
	Moreover, let us define the set $\Xi_{\epsilon,\bar g}\coloneqq\{\vxi \in \Xi: \max_{i \in [m]} \bar{g}_i(\vxi)\ge \epsilon\}$.
	
	As a first step, we prove that $\bar \vxi \in \Xi_{\epsilon,\bar g}$. In particular, since the event $\CES$ holds, we have that
	\begin{align*}
	\max_{i \in [m]} \bar{g}_i(\bar \vxi)\ge \frac{\sum_{t=1}^\tau \max_{i \in [m]} {g}_i(\bar \vxi)- \err_{\tau,\eta}}{\tau}=\epsilon
	\end{align*}
	For every $\tau \in [T_1]$, we have:
	{\allowdisplaybreaks\begin{subequations}
	\begin{align}
	\sum_{t=1}^\tau  \cL_{f_t,g_t}(\vx_t,\vlambda_t) & \le \sum_{t=1}^\tau \cL_{f_t,g_t}(\vx_t,\vlambda^*) + \frac{1}{ \tilde \rho} \cumd_\tau \label{long:one}\\
	& \hspace{-1cm} \le \sum_{t=1}^\tau \cL_{\bar f,\bar g}(\vx_t,\vlambda^*) +\frac{1}{\tilde \rho} \cumd_\tau+\mleft(1+\frac{1}{\tilde \rho}\mright) \err_{\tau,\eta}\label{long:two} \\
	& \hspace{-1cm} \le  \inf_{\vlambda \in \cD_{ \tilde \rho}} \sum_{t=1}^\tau \cL_{\bar f, \bar g}(\vx_t,\vlambda) +\frac{1}{\tilde \rho} \cumd_\tau+\mleft(1+\frac{1}{\tilde \rho}\mright)  \err_{\tau,\eta} \label{long:three}\\
	& \hspace{-1cm} =  \tau \inf_{\vlambda \in \cD_{ \tilde \rho}}  \cL_{\bar f,\bar g}(\bar \vxi,\vlambda) +\frac{1}{\tilde \rho} \cumd_\tau+\mleft(1+\frac{1}{\tilde\rho}\mright)  \err_{\tau,\eta}\label{long:four}\\
	& \hspace{-1cm} =  \tau \inf_{\vlambda \in \cD_{ \rho/2}}  \cL_{\bar f,\bar g}(\bar \vxi,\vlambda) +\frac{1}{\tilde \rho} \cumd_\tau+\mleft(1+\frac{1}{\tilde\rho}\mright)  \err_{\tau,\eta}\label{long:boooo}\\
	& \hspace{-1cm} \le  \tau \sup_{\vxi \in \Xi_{\epsilon,\bar g}}\inf_{\lambda \in \cD_{ \rho/2}} \cL_{\bar f, \bar g}(\vxi,\vlambda) +\frac{1}{\tilde \rho} \cumd_\tau+\mleft(1+\frac{1}{\tilde\rho}\mright) \err_{\tau,\eta} \label{long:six} \\
	& \hspace{-1cm}\le \tau \sup_{\vxi \in \Xi}\inf_{\vlambda \in \cD_{ \rho}} \mleft(\cL_{\bar f, \bar g}(\vxi,\vlambda) - \frac{\epsilon}{ \rho}\mright) +\frac{1}{\tilde \rho} \cumd_\tau+\mleft(1+\frac{1}{\tilde \rho}\mright)  \err_{\tau,\eta}  \label{long:seven} \\
	& \hspace{-1cm}= \tau\mleft( \OPT_{\bar f, \bar g} -\frac{\epsilon}{  \rho}\mright) +\frac{1}{\tilde \rho} \cumd_\tau+\mleft(1+\frac{1}{\tilde \rho}\mright)  \err_{\tau,\eta} \label{long:nine}   \\
	& \hspace{-1cm}= \tau\OPT_{\bar f, \bar g} -\tau \frac{\max_{i' \in [m]}\sum_{t =1}^\tau g_{t,i'}(\vx_t)- \err_{\tau,\eta}}{\tau  \rho} +\frac{1}{\tilde \rho} \cumd_\tau+\mleft(1+\frac{1}{ \tilde \rho}\mright)  \err_{\tau,\eta} \label{long:ten} \\
	& \hspace{-1cm}\leq \tau\OPT_{\bar f, \bar g} +\frac{1}{\tilde \rho} \cumd_\tau+\mleft(1+\frac{2}{\tilde \rho}\mright)  \err_{\tau,\eta} - \max_{i' \in [m]}\frac{\sum_{t =1}^\tau g_{t,i'}(2\vx_t)}{  \rho} \label{long:eleven} \\
	& \hspace{-1cm}\leq \tau\OPT_{\bar f, \bar g} +\frac{1}{\tilde \rho} \cumd_\tau+\mleft(1+\frac{2}{\tilde \rho}\mright)  \err_{\tau,\eta} - \max_{i' \in [m]} \sum_{t =1}^\tau g_{t,i'}(\vx_t) \label{long:twelve}\\
	& \hspace{-1cm} \leq \tau\OPT_{\bar f, \bar g} +\frac{1}{\tilde \rho} \cumd_\tau+\mleft(1+\frac{2}{\tilde \rho}\mright)  \err_{\tau,\eta} - \sum_{t =1}^\tau g_{t,i}(\vx_t) \qquad \qquad \qquad \quad \forall i \in [m],\label{long:thirteen}
	\end{align}
	\end{subequations}}%
	where Equation~\eqref{long:one} is given by the no-regret property of the dual regret minimizer, and Equation~\eqref{long:two} by the definition of the event $\bar \CE$, which holds by assumption.
	Moreover, Equation~\eqref{long:four} follows from the fact that $\bar f$ and $\bar g$ are independent from $t$, Equation~\eqref{long:boooo} follows from $\tilde \rho =\hat \rho /2\le  \rho /2 $, and Equation~\eqref{long:six} from $\bar \vxi \in \Xi_{\epsilon,\bar g}$.
	Finally, Equation~\eqref{long:seven} follows from Lemma~\ref{lm:FarFromOpt}, Equation~\eqref{long:ten} by definition of $\epsilon$, and Equation~\eqref{long:eleven} by $\tilde \rho \le \rho$.
\end{proof}

\theoremStoc*

\begin{proof}
	We prove the statement of the theorem by considering two cases.
	
	\paragraph{Case ``Condition \ref{ass:strictly stoc} holds''.}
	By Lemma \ref{lm:ce}, event $\CE$ holds with probability at least $1 - \delta$. 
	In the rest of the proof, we assume that the event $\CE$ holds, and we provide a bound that holds with probability at least $1-\delta$.
	For every $\tau \in [T_1]$, we have: 
	\begin{align*}
		\sum_{t = 1}^{\tau} g_t(\vx_t) &\le   \tau \OPT_{\bar f, \bar g} -  \sum_{t = 1}^{\tau} \cL_{f_t, g_t}(\vx_t,\vlambda_t) +\frac{1}{\tilde \rho} \cumd_\tau+\mleft(1+\frac{2}{\tilde \rho}\mright) \err_{\tau,\eta}  \\
		&   \le \mleft(1+\frac{2}{\tilde \rho}\mright)\cump_{\tau,\eta}+\mleft(1+\frac{1}{\tilde\rho}\mright) \err_{\tau,\eta} +\frac{1}{\tilde \rho} \cumd_\tau+\mleft(1+\frac{2}{\tilde\rho}\mright) \err_{\tau,\eta}\\
		&= \mleft(2+\frac{3}{\tilde \rho}\mright) \err_{\tau,\eta} + \mleft(1+\frac{2}{\tilde \rho}\mright)\cump_{\tau,\eta} +\frac{1}{\tilde \rho} \cumd_\tau \\
		&\le \frac{2}{\tilde \rho} \sqrt{T} -1 + \mleft(2+\frac{3}{\tilde \rho}\mright) \err_{\tau,\eta} + \mleft(1+\frac{2}{\tilde \rho}\mright)\cump_{\tau,\eta} +\frac{1}{\tilde \rho} \cumd_\tau \\
		&= M_{\tilde \rho} - 1,
	\end{align*}

where the first inequality follows from Lemma \ref{lm:stocstocuno}, the second one from Lemma \ref{lm:stocstocdue}, the third one from the fact that $\frac{2}{\tilde \rho} \sqrt{T} -1 \geq 0$, being $\tilde \rho \leq 1$, and the last equation follows from the definition of $M_{\tilde \rho}$.
This implies that the algorithm never enters the recovery phase when Condition \ref{ass:strictly stoc} holds.

\paragraph{Case ``Condition \ref{ass:strictly stoc} does \emph{not} hold''.}
By Lemma \ref{lm:ce}, event $\CE$ holds with probability at least $1 - \delta$. 
In the rest of the proof, we assume that the event $\CE$ holds, and we provide a bound that holds with probability at least $1-\delta$.
	Suppose by contradiction that $T_1 <T$.
	This implies that a constraint $i \in [m]$ is violated by at least $M_{T^{-1/4}} - 1$.
	Let $i^\star \in \argmax_{i \in [m]} \sum_{t = 1}^{T_1} g_{t,i}(\vx_t)$ be one of the most violated constraints during the play phase.
	Then, we have:
	{\allowdisplaybreaks\begin{align*}
		\sum_{t=1}^{T_1}  \cL_{f_t,g_t}(\vx_t,\vlambda_t)& = \sum_{t=1}^{T_1} \Big( f(\vx_t)- \langle \vlambda_t , g_t(\vx_t)\rangle \Big)\\
		&\le T_1 - \sum_{t=1}^{T_1} \langle \vlambda_t, g_t(\vx_t) \rangle \\
		&\le T_1 - \sum_{t=1}^{T_1} \frac{1}{T^{-1/4}} g_{t,i^\star}(\vx_t)+ T^{1/4} \cumd_{T_1}  \\
		& \le T_1- T^{1/4} (M_{T^{-1/4}}-1) + T^{1/4} \cumd_{\tau_1}  \\
		&< - \mleft(1+\frac{2}{T^{-1/4}}\mright)\cump_{\tau,\eta}- \frac{1}{T^{-1/4}}\err_{\tau,\eta},
	\end{align*}}%
	where the second inequality follows from the no-regret property of the dual regret minimizer and the fact that, when Condition \ref{ass:strictly stoc} does \emph{not} hold, $\tilde \rho = T^{-1/4}$. 
	The last inequality follows from the definition of $M_{T^{-1/4}}$.
	Then, the result above allows us to reach the desired contradiction when compared with the following one.
	In particular, for every $\tau \in [T_1]$, we have:
	\begin{align*}
		\sum_{t=1}^\tau  \cL_{f_t,g_t}(\vx_t,\vlambda_t) & \ge \sum_{t=1}^\tau  \cL_{f_t,g_t}(\vxi^\circ,\vlambda_t)- \mleft(1+\frac{2}{T^{-1/4}}\mright) \cump_{\tau,\eta} \\
		&\ge \sum_{t=1}^\tau  \cL_{f_t,\bar g}(\vxi^\circ,\vlambda_t)- \frac{1}{T^{-1/4}}\err_{\tau,\eta} - \mleft(1+\frac{2}{T^{-1/4}}\mright) \cump_{\tau,\eta}\\
		&\ge - \frac{1}{T^{-1/4}}\err_{\tau,\eta} - \mleft(1+\frac{2}{T^{-1/4}}\mright) \cump_{\tau,\eta} ,
	\end{align*}
	where the first inequality follows from the no-regret property of the primal regret minimizer, the second one follows from the fact that event $\CE$ holds, and the third one from the feasibility of $\vxi^\circ$.
\end{proof}

\subsection{ Proofs omitted from Section~\ref{sec:adv}}

As a first step, we provide a lower bound for the cumulative reward achieved during the play phase.
In particular, we show that it achieves at least a ${\rho}/{(1+\rho)}$ fraction of the value obtained by an optimal solution in the first $T_1$ rounds. 

\begin{restatable}{lemma}{lemmaSmallRegretAdv}\label{lm:smallRegAdv}
	If Condition $\ref{ass:strictly stoc}$ is satisfied, then, with probability at least $1-\eta$, at round $T_1$ of Algorithm~\ref{alg:meta alg known} it holds that: 
	\begin{align*}
		\sum_{t = 1}^{T_1} f_t(\vx_t) \ge \frac{\rho}{1+\rho} \sum_{t = 1}^{T_1}  f_t(\vxi^*) +\mleft(T-T_1\mright) - \mleft(1+\frac{2}{\tilde\rho}\mright)\cump_{T_1,\eta}- \frac{1}{\tilde \rho}\cumd_{\tau_1}.  
	\end{align*}
\end{restatable}

\begin{proof}
	Let $\bar \vxi \in \Xi$ be a strategy mixture obtained by playing with probability $1/(1+\rho)$ the mixture $\vxi^\circ$ and with the remaining probability $\rho/(1+\rho)$ an optimal mixture $\vxi^*$.
	Notice that the probabilities are well defined, since $\rho \ge0$. Then, for every $t \in [T]$ and $i \in [m]$, it holds:
	\begin{align*}
	\frac{1}{1+\rho}g_{t,i}(\vxi^\circ) + \frac{\rho}{1+\rho}g_{t,i}(\vxi^*) \leq -\frac{\rho}{1+\rho} + \frac{\rho}{1+\rho} = 0
	\end{align*}
	where the inequality follows from the fact that $g_{t,i}(\vxi^\circ)\le -\rho$ and $g_{t,i}(\vxi^*)\le 1$.
	Therefore, for every $t \in [T]$ and $i \in [m]$, it holds that $g_t(\bar \vxi)\le 0$.
	Assume that the regret bounds of the regret minimizers hold. Notice that this happens with probability at least $1-\eta$. Then, by the no-regret property of the primal regret minimizer, we have that 
	\begin{align} \label{eq:reg prim adv}
	\sum_{t = 1}^{T_1} \cL_{f_t,g_t}(\vx_t,\vlambda_t) \ge \sum_{t = 1}^{T_1}\cL_{f_t,g_t}(\bar \vxi,\vlambda_t)-\mleft(1+\frac{2}{\tilde\rho}\mright)\cump_{T_1,\eta}.
	\end{align}
	
	Let $i^\star \in \argmax_{i \in [m]} \sum_{t =1}^{T_1} g_{t,i}(\vx_t)$ be one of the most violated constraints during the play phase.
	Next, we prove that 
	\begin{align*}%
	\sum_{t =1}^{T_1} \langle \vlambda_t, g_t(\vx_t) \rangle \ge (T-T_1) - \frac{1}{\tilde \rho} \cumd_{T_1}.
	\end{align*}
	We consider two cases. If $T_1=T$, then 
	\begin{align*}
	\sum_{t = 1}^{T_1} \langle \vlambda_t, g_t(\vx_t) \rangle \ge \sum_{t = 1}^{T_1}  \langle \mathbf{0}, g_{t}(\bar \vxi)\rangle - \frac{1}{\tilde \rho} \cumd_{T_1} = - \frac{1}{\tilde \rho} \cumd_{\tau_2} =  (T-T_1) - \frac{1}{\tilde \rho} \cumd_{T_1}.
	\end{align*}
	Otherwise, we have that $ \sum_{t = 1}^{T_1} g_{t,i^\star}(\vx_t) \ge \tilde \rho(T-T_1)$ and 
	\begin{align} \label{eq:dual guarantees}
	\sum_{t = 1}^{T_1} \langle \vlambda_t, g_{t}( \vx_t) \rangle \ge \mleft( \sum_{t = 1}^{T_1} \frac{1}{\tilde \rho} g_{t,i^\star}(\vx_t) \mright) - \frac{1}{\tilde \rho} \cumd_{T_1}\ge (T-T_1) - \frac{1}{\tilde \rho} \cumd_{T_1}.
	\end{align}
	Thus,
	\begin{align*}
	\sum_{t = 1}^{T_1} f_t(\vx_t) &\ge \sum_{t = 1}^{T_1} \Big( f_t(\bar \vxi) - \langle \vlambda_t, g_t(\bar \vxi) \rangle + \langle \vlambda_t, g_t(\vx_t) \rangle \Big) -\mleft(1+\frac{2}{\tilde\rho}\mright)\cump_{T_1,\eta}\\
	&\ge \sum_{t = 1}^{T_1} \Big( f_t(\bar \vxi) - \langle \vlambda_t, g_t(\bar \vxi)\rangle \Big) +(T-T_1) -\mleft(1+\frac{2}{\tilde\rho}\mright)\cump_{T_1,\eta}-  \frac{1}{\tilde \rho}\cumd_{\tau_2}\\
	&\ge \sum_{t = 1}^{T_1}  f_t(\bar \vxi)  +(T-T_1) -\mleft(1+\frac{2}{\tilde\rho}\mright)\cump_{T_1,\eta}-  \frac{1}{\tilde \rho}\cumd_{\tau_2}\\
	&\ge \sum_{t = 1}^{T_1} \mleft(  \frac{1}{1+\rho}f_t(\vxi^\circ) +  \frac{\rho}{1+\rho}f_t(\vxi^*) \mright) +(T-T_1) -\mleft(1+\frac{2}{\tilde\rho}\mright)\cump_{T_1,\eta}-  \frac{1}{\tilde \rho}\cumd_{\tau_2}\\
	&\ge \frac{\rho}{1+\rho} \sum_{t = 1}^{T_1}   f_t(\vxi^*)  +(T-T_1) -\mleft(1+\frac{2}{\tilde\rho}\mright)\cump_{T_1,\eta}- \frac{1}{\tilde \rho}\cumd_{T_1},
	\end{align*}
	where the first inequality follows from Equation~\eqref{eq:reg prim adv}, the second one from Equation~\ref{eq:dual guarantees}, the third one from the fact that for each $t \in [T]$ it holds $g_t(\bar \vxi)\le 0$, while the fourth inequality follows from the definition of $\bar \vxi$.
	This concludes the proof.
\end{proof}
Notice that, for a small $T_1$, we have a large lower bound on the cumulative reward. Intuitively, this means that when the play phase is short, the primal regret minimizer accumulated so much regret in the play phase that the recovery phase can be addressed without worrying about the reward.

As a second step, we provide an upper bound on the cumulative constraints violation during the recovery phase. 
In particular, we show that the constraints are satisfied by at least $\rho$ at each round up to a term related to the regret of $\cRp$ and $\cRd$.
\begin{restatable}{lemma}{lemmaSmallViolationAdv}\label{lm:smallViolationAdv}
	With probability at least $1-\eta$, when Algorithm~\ref{alg:meta alg known} halts it holds that for each $i \in [m]$:
	\begin{align*}
		\sum_{t = T_1+1}^T g_{t,i}(\vx_t)\le -(T-T_1) \rho + \cumd_{ T-T_1} + 2 \cump_{T-T_1, \eta}.
	\end{align*}
\end{restatable}

\begin{proof}
	Let $i^\star$ be one of the most violated constraints, \emph{i.e.,} $i^\star \in \argmax_{i\in [m]} \sum_{t = T_1 +1}^{T} g_{t,i}(\vx_t)$. Then, we have that:
	\begin{align*}
	(T-\tau)\rho & \le -\sum_{t = T_1 +1}^{T} \langle \vlambda_t,  g_t(\vxi^\circ) \rangle  \\
	&\le - \sum_{t = T_1 +1}^{T} \langle \vlambda_t ,g_t(\vx_t) \rangle+ 2 \cump_{T-T_1,\eta} \\
	&\le  - \sum_{t = T_1 +1}^{T}  g_{t,i^\star}(\vx_t) +\cumd_{ T-T_1} + 2 \cump_{T-T_1,\eta},
	\end{align*}
	where the first inequality follows from the definition of $\vxi^\circ$ and the fact that it is always feasible of at least $\rho$, the second one follows from the assumption that the primal regret minimizer satisfies the regret bound, and the last inequality from the guarantee on the regret of the dual regret minimizer.
	We conclude the proof by noticing that the regret bound of the primal regret minimizer holds with probability at least $1-\eta$.
\end{proof}

Now, we can provide our bounds for adversarial constraints.

\theoremAdv*
\begin{proof}
		In the following, we assume that both Lemma \ref{lm:smallRegAdv} and Lemma \ref{lm:smallViolationAdv}. By an union bound, this holds with probability $1-2\eta=1-\frac{2}{3} \delta$.
	Then, it  holds 
	\begin{align*}
		\sum_{t= 1}^{T} f_t(\vx_t) & \ge \sum_{t=1}^{T_1} f_t(\vx_t) \\
		& \ge  \sum_{t=1}^{T_1}  \frac{\rho}{1+\rho}f_t(\vxi^\star) +(T-T_1) -\mleft(1+\frac{2}{\tilde\rho}\mright)\cump_{T_1,\eta}- \frac{1}{\tilde \rho}\cumd_{T_1} \\
		&\ge \frac{\rho}{1+\rho} \sum_{t=1}^{T} f_t(\vxi^\star)   -\mleft(1+\frac{2}{\tilde\rho}\mright)\cump_{T_1, \eta}- \frac{1}{\tilde \rho}\cumd_{T_1}\\
		&\ge \frac{\rho}{1+\rho} \sum_{t=1}^{T} f_t(\vxi^\star)   -\mleft(1+\frac{2}{\tilde\rho}\mright)\cump_{T, \eta}- \frac{1}{\tilde \rho}\cumd_{T},
	\end{align*}
where the second inequality comes from Lemma \ref{lm:smallRegAdv}. This proves the bound on the regret.
	
	By Lemma \ref{lm:smallViolationAdv}, for each $i \in [m]$, 
	\begin{align*}
		\sum_{t =1}^{T} g_{t,i}(\vx_t) &= 
		\sum_{t =1}^{T_1} g_{t,i}(\vx_t) +\sum_{t = T_1+1}^{T} g_{t,i}(\vx_t) \\
		&\le (T-T_1) \tilde \rho + M_{\tilde \rho} -(T-T_1) \rho + \cumd_{ T-T_1} + 2 \cump_{T-T_1,\eta}\\
		&\le M_{\tilde \rho}+  \cumd_{ T-T_1} + 2 \cump_{T-T_1,\eta}\\
		&\le M_{\tilde \rho}+  \cumd_{ T} + 2 \cump_{T,\eta},
	\end{align*}
	where the second inequality comes from $ \tilde \rho\le \rho$.
	
\end{proof}

\corollaryAdv*

\begin{proof}
	It is easy to see that Theorem \ref{thm:advAdv} can be extended to consider the definition of $\vxi^\star$ for stochastic rewards.
	formally, it holds $\sum_{t= 1}^{T} f_t(x_t)\ge \frac{\rho}{1+\rho} \sum_{t=1}^{T} f_t(\vxi^\star)   -\mleft(1+\frac{2}{\tilde\rho}\mright)\cump_{T, \eta}- \frac{1}{\tilde \rho}\cumd_{T}$.
	Consider the two martingale difference sequences $\sum_{t=1}^T f_t(\vx_t)- \bar f(\vx_t)$ and $\sum_{t} f_t(\vxi^\star)- \bar f(\vxi^\star)$. We can apply Azuma-Hoeffding inequality to prove that, with probability at least $1-\eta$, it holds $\sum_{t} |f_t(\vx_t)- \bar f(\vx_t)|\le \err_{T,\eta}$ and $\sum_{t} |f_t(\vxi^\star)- \bar f(\vxi^\star)|\le \err_{T,\eta}$.
	Then, 
	\begin{align*}
		\sum_{t = 1}^{T} \bar f(\vx_t) &\ge \sum_{t = 1}^{T}  f_t(\vx_t) -\err_{T,\eta} \\
		&\ge  \frac{\rho}{1+\rho} \sum_{t = 1}^{T} f_t(\vxi^\star)  -\mleft(1+\frac{2}{\tilde \rho}\mright)\cump_{T,\eta}- \frac{1}{\tilde \rho} \cumd_T - \err_{T,\eta}\\
		&\ge \frac{\rho}{1+\rho} \sum_{t = 1}^{T} \bar f(\vxi^\star)  -\mleft(1+\frac{2}{\tilde \rho}\mright)\cump_{T,\eta}- \frac{1}{\tilde \rho} \cumd_T - 2 \err_{T,\eta}, 
	\end{align*}
	proving the statement.
\end{proof}

\subsection{Proofs omitted from \Cref{sec:unknown}}

\lemmaGoodEstimator*
\begin{proof}
	By Azuma-Hoeffding inequality, we have that with probability at least $1-\delta$, for each $i \in [m]$ it holds $\big |\sum_{t = {1}}^{T_0}  g_{t,i}(\vx_t)-\bar g_i(\vx_t)\big|$.
	Hence, 
	\begin{align*}
		- \max_{i \in [m]}\sum_{t = 1}^{T_0}  g_t(\vx_t) \le - \max_{i \in [m]}\sum_{t = 1}^{T_0}   \bar g(\vx_t) + \err_{T_0,\delta}\le T_0  \bar g(\vxi^\circ) + \err_{T_0,\delta}=T_0\rho + \err_{T_0,\delta},
	\end{align*}
	where the second and third inequality follow from the definition of $\vxi^\circ$.
	Then, 
	\begin{align*}
		\hat  \rho &= - \frac{1}{ T_0} \left( \max_{i \in [m]}\sum_{t=1}^{T_0}g_{t,i}(\vx_t) + \err_{ T_0,\delta} \right)\\
		&\le \frac{1}{ T_0} \left( T_0 \rho + \err_{ T_0,\delta}- \err_{ T_0,\delta} \right) \\
		&= \rho.
	\end{align*}
	This concludes the proof.
\end{proof}

\lemmaSuperGoodEstimator*

\begin{proof}
	First, notice that with probability $1-\delta$, the primal regret minimizer has regret bounded by $\cump_{T_0,\delta}$. Moreover, by the Azuma-Hoeffding inequality, it holds $\left|\sum_{t=1}^{T_0} \vlambda_t g_t(\vxi^\circ) - \vlambda_t \bar g(\vxi^\circ)\right| \le \err_{T_0,\delta}$ with probability $1-\delta$.
	Consider the case in which both the conditions hold. This happens with probability at least $1-2 \delta$ by a union bound.
	
	Then,  \begin{align*}
		- \max_{i\in [m]} \sum_{t=1}^{T_0} g_t(\vx_t) &\ge - \sum_{t=1}^{T_0} \langle \vlambda_t, g_t(\vx_t) \rangle - \cumd_{T_0}\\
		&\ge - \sum_{t=1}^{T_0} \langle \vlambda_t, g_t(\vxi^\circ) \rangle - \cumd_{T_0}- 2\cump_{T_0,\delta} \\
		&\ge -\sum_{t=1}^{T_0} \langle \vlambda_t , \bar g(\vxi^\circ) \rangle - \cumd_{T_0}- 2\cump_{T_0,\delta}- \err_{T_0,\delta}\\
		& \ge  T_0 \rho - \cumd_{T_0}- 2\cump_{T_0,\delta}- \err_{T_0,\delta}.
	\end{align*}

	Hence, \begin{align*}
		\hat \rho &= - \frac{1}{ T_0} \left( \max_{i \in [m]}\sum_{t=1}^{T_0}g_{t,i}(\vx_t) + \err_{ T_0,\delta} \right)  \\
		&\ge  \frac{1}{ T_0} \left( T_0 \rho- \cumd_{T_0}- 2\cump_{T_0,\delta}- \err_{T_0,\delta}- \err_{ T_0,\delta} \right)\\
		& \ge \rho/2 + \frac{1}{ T_0} \left( T_0 \rho/2- \cumd_{T_0}- 2\cump_{T_0,\delta}-  2 \err_{T_0,\delta}\right)\\ &\ge  \rho/ 2,
	\end{align*}
	where the last inequality comes from Condition~\ref{ass:strictly}. This concludes the proof.
\end{proof}

\end{document}